\documentclass[twoside]{article}
\usepackage[accepted]{aistats2021}

\usepackage[utf8x]{inputenc}

\usepackage{natbib}
\usepackage{hyperref}
\usepackage{url}
\usepackage{booktabs}
\usepackage{color}
\usepackage{placeins}
\usepackage{amsmath}
\usepackage{amsthm}
\usepackage{amssymb}
\usepackage{caption}
\usepackage{subcaption}
\usepackage{tikz}
\usepackage{comment}
\usepackage{bm}
\usepackage{array}
\usepackage{scalefnt}
\usepackage{mathtools}
\usepackage[ruled,vlined,noend]{algorithm2e}
\usepackage{footnote}
\usepackage{amsmath}
\usepackage{amssymb}
\usepackage{graphicx}
\usepackage[colorinlistoftodos]{todonotes}
\usepackage{thm-restate}

\usepackage{enumitem}
\newcommand{\Ac}{{\mathcal{A}}}
\newcommand{\Bc}{{\mathcal{B}}}

\newcommand{\Dc}{{\mathcal{D}}}

\newcommand{\Ic}{{\mathcal{I}}}

\newcommand{\Sc}{{\mathcal{S}}}

\newcommand{\Vc}{{\mathcal{V}}}

\newcommand{\Xc}{{\mathcal{X}}}
\newcommand{\Yc}{{\mathcal{Y}}}

\newcommand{\Eb}{{\mathbb{E}}}

\newtheorem{prop}{Proposition}

\newtheorem{definition}{Definition}

\newtheorem{lemma}{Lemma}

\newcommand\numberthis{\addtocounter{equation}{1}\tag{\theequation}}

\newcommand\blfootnote[1]{%
  \begingroup
  \renewcommand\thefootnote{}\footnote{#1}%
  \addtocounter{footnote}{-1}%
  \endgroup
}
\renewcommand{\cite}{\citep}

\newcommand{\tmu}{{\tilde{\mu}}}

\newcommand{\hc}{{\hat{c}}}
\newcommand{\tg}{{\tilde{b}}}

\newcommand{\rswapt}{{R_T^{\mathrm{swap}}}}

\newcommand{\rswapdt}{{\tilde{R}_T^{\mathrm{swap}}}}

\begin{document}

\runningtitle{A Mechanism Design Alternative to Individual Calibration}

\twocolumn[

\aistatstitle{Right Decisions from Wrong Predictions: \\ A Mechanism Design Alternative to Individual Calibration}

\aistatsauthor{ Shengjia Zhao \And Stefano Ermon }

\aistatsaddress{ Stanford University \And Stanford University } ]

\begin{abstract}
Decision makers often need to rely on imperfect probabilistic forecasts. While average performance metrics are typically available, it is difficult to assess the quality of individual forecasts and the corresponding utilities. To convey confidence about individual predictions to decision-makers, we propose a compensation mechanism ensuring that the forecasted utility matches the actually accrued utility. While a naive scheme to compensate decision-makers for prediction errors can be exploited and might not be sustainable in the long run, we propose a mechanism based on fair bets and online learning that provably cannot be exploited. We demonstrate an application showing how passengers could confidently optimize individual travel plans based on flight delay probabilities estimated by an airline.
%

\end{abstract}

\section{Introduction}

People and 
algorithms 
constantly  rely  on  probabilistic  forecasts  (about medical treatments, weather, transportation times, etc.)  and  make potentially high-stake decisions based on them. 
In most cases, forecasts are not perfect, e.g., the forecasted chance that 
it will rain tomorrow
does not match the true probability exactly.
While average performance statistics might be available (accuracy, calibration, etc), it is generally impossible to tell whether any \emph{individual} prediction is reliable (individually calibrated), 
e.g., 
about the medical condition of an \emph{specific patient} or 
the delay of a \emph{particular flight}
~\cite{vovk2005algorithmic,barber2019limits,zhao2020individual}. 
Intuitively, this is because multiple \emph{identical} datapoints are needed to confidently estimate a probability from empirical frequencies, but identical datapoints are rare in real world applications (e.g. two patients are always different).
Given these limitations, we study alternative mechanisms to convey confidence about \emph{individual} predictions to decision-makers. 


We consider settings where 
a single forecaster provides predictions to 
many 
decision makers, each facing a potentially different decision making problem. 
For example, a personalized medicine service could predict whether a  product is effective for thousands of individual patients~\cite{ng2009agenda,pulley2012operational,bielinski2014preemptive}. 
If the prediction is accurate for 70\% of patients, it could be accurate for Alice but not Bob, or vice-versa. 
Therefore, Alice might be hesitant to make decisions based on the 70\% \textit{average} accuracy. 
In this setting, we 
propose an insurance-like mechanism that 1) 
enables each 
decision maker to confidently make decisions as if the advertised probabilities were individually correct, and 2) is implementable by the forecaster with provably vanishing costs in the long run. 
%
%
%
%
%
%
%
%

To achieve this, we turn to the classic idea~\cite{de1931sul,jaynes1996probability} that a probabilistic belief is equivalent to a willingness to take bets. 
We use the previous example to illustrate that if the forecaster is willing to take bets, a decision maker can bet with the forecaster as an ``insurance'' against mis-prediction. 
Suppose Alice 
is trying to decide whether or not to use a product. 
If she uses the product, 
she gains \$10 if the product is effective and loses \$2 otherwise. 
The personalized medicine service (forecaster) predicts that the product is effective with 50\% chance for Alice. Under this probability Alice expects to gain \$4 if she decides to use the product, but she is worried the probability is incorrect. Alice proposes a bet: Alice pays the forecaster \$6 if the product is effective, and the forecaster pays Alice \$6 otherwise. The forecaster should accept the bet because under its own forecasted probability the bet is fair (i.e., the expectation is zero 
if the forecasted probabilities
are true for Alice). 
Alice gets the guarantee that if she decides to use the product, \emph{effective or not}, she gains \$4 ---  equal to 
her expected utility under the forecasted (and possibly incorrect) probability. In general, we show that \emph{Alice has a way of choosing bets for any utility function and forecasted probability, such that her true gain 
equals her expected gain under the forecasted probability}.


From the forecaster's perspective, if the true probability that Alice's treatment is effective is actually 10\%, then the forecaster will lose \$4.8 from this bet in expectation. However, in our setup, the forecaster makes probabilistic forecasts for many different decision makers, 
each selecting some bet based on their utility function and forecasted probability.  
The forecaster might gain or lose on \textit{individual} bets, but it only needs to not lose on the entire set of bets \textit{on average} for the approach to be sustainable. 
Intuitively, our mechanism averages individual decision maker's 
difference between forecasted gain and true gain 
so the difficult requirement that \textit{each} difference should be negative has been reduced to an easier requirement that the \textit{average} difference should be negative.  
However, this protocol leaves the forecaster vulnerable to exploitation. For example, if Alice already knows that the product will be ineffective; she could still bet with the forecaster for the malicious purpose of gaining \$6. Surprisingly we show that in the online setup~\cite{cesa2006prediction}, 
the forecaster has an algorithm to adapt its forecasts and guarantee vanishing loss in the long run, even in the presence of malicious decision makers. 
This is achieved by first using any existing online prediction algorithm to predict the probabilities, then applying a post processing algorithm to fine-tune these probabilities based on past gains/losses (similar to the idea of recalibration~\cite{kuleshov2017estimating,guo2017calibration}). 


As a concrete application of our approach, we simulate the interaction between an airline and passengers with real flight delay data. Risk averse passengers might want to avoid a flight if there is possibility of delay and their loss in case of delay is high. We show if an airline offers to accept bets based on the predicted probability of delay, it can help risk-averse passengers make better decisions, and increase both the airline's revenue (due to increased demand for the flight) and the total utility (airline revenue plus passenger utility). 

We further verify our theory with large scale simulations on several datasets and a diverse benchmark of decision tasks. We show that forecasters based on our post-processing algorithm consistently achieve close to zero betting loss (on average) within a small number of time steps. On the other hand, several seemingly reasonable alternative algorithms not only lack theoretical guarantees, but often suffer from positive average betting loss in practice.

\section{Background}

\subsection{Decision Making with Forecasts}
This section defines the basic setup of the paper. We represent the decision making process as a multi-player game between nature, a forecaster and a set of (decision making) agents. At every step $t$ nature reveals an input observation $x_t$ to the forecaster (e.g. patient medical records)
and selects the hidden probability $\mu_t^* \in [0, 1]$ that $\Pr[y_t=1]=\mu_t^*$ (e.g. probability treatment is successful),
We only consider binary variables ($y_t \in \lbrace 0, 1\rbrace = \Yc$) and defer the general case to Appendix B. 

The forecaster chooses a forecasted probability $\mu_t \in [0, 1]$ to approximate $\mu_t^*$. We also allow the forecaster to represent the lack of knowledge about $\mu_t^*$, i.e. the forecaster outputs a confidence $c_t \in [0, 1]$ where the hope is that $\mu_t^* \in [\mu_t - c_t, \mu_t + c_t]$. 

At each time step, one or more agents can use the forecast $\mu_t$ and $c_t$ to make decisions, i.e. to select an action $a_t \in \Ac$. 
However, for simplicity we assume that different agents make decisions at different time steps, so at each time step there 
is only a single agent, and we can uniquely index the agent by the time step $t$. 
The agent knows its own loss (negative utility) function $l_t: \Ac \times \Yc \to [-M, M]$ (the forecaster does not have to know this) where $M \in \mathbb{R}_+$ is the maximum possible loss. 
This protocol is formalized below. 

\paragraph{Protocol 1: Decision Making with Forecasts} For $t = 1, \cdots, T$
\begin{enumerate}[topsep=0pt,itemsep=0.1ex,partopsep=0ex,parsep=0ex]
    \item Nature reveals $x_t \in \Xc$ to forecaster and chooses $\mu_t^* \in [0, 1]$ without revealing it
    \item Forecaster reveals $\mu_t, c_t \in (0, 1)$ where $(\mu_t-c_t, \mu_t+c_t) \subset (0, 1)$
    \item Agent $t$ has loss function $l_t: \Ac \times \Yc \to \mathbb{R}$ and reveals $a_t$ selected according to $\mu_t, c_t$ and $l_t$
    \item Nature samples $y_t \sim \mathrm{Bernoulli}(\mu_t^*)$ and reveals $y_t$; Agent incurs loss $l_t(a_t, y_t)$
\end{enumerate}

We make no assumptions on nature, forecaster, or the agents.  
They can choose any strategy to generate their actions, as long as they do not look into the future (i.e. their action only depends on variables that have already been revealed). In particular, we make no i.i.d. assumptions on how nature selects $y_t$ and $\mu_t^*$; for example, nature could even select them adversarially to maximize the agent's loss. 




%


\subsection{Individual Coverage}
\label{sec:impossibility} 

Ideally in Protocol 1 the forecaster's prediction $\mu_t, c_t$ should satisfy $\mu_t^* \in [\mu_t - c_t, \mu_t + c_t]$) for each individual $t$ (this is often called individual coverage or individual calibration in the literature). 
However, many existing results show that learning individually calibrated probabilities from past data is often impossible~\cite{vovk2005algorithmic,barber2019limits,zhao2020individual} unless the forecast is trivial (i.e. $[\mu_t - c_t, \mu_t + c_t] = [0, 1]$). 

One intuitive reason for this impossibility result is that in many practical scenarios
for each $x_t$ we only observe a single sample $y_t \sim \mu_t^*$. The forecaster cannot 
infer $\mu_t^*$ from a single sample $y_t$ without relying on unverifiable assumptions. 


\subsection{Probability as Willingness to Bet}
A major justification for probability theory has been that probability can represent willingness to bet~\cite{de1931sul,halpern2017reasoning}.
For example, if you truly believe that a coin is fair, then it would be inconsistent if you are not willing to 
win \$1 for heads, and lose \$1 for tails (assuming you only care about average gain rather than risk).  
More specifically a forecaster that holds a probabilistic belief should be willing to accept any bet where it 
gains a non-negative 
amount in expectation. 

For binary variables, we consider the case where a forecaster believes that a binary  event $Y \in \lbrace 0, 1 \rbrace$ happens with some probability $\mu^*$ but does not know the exact value of $\mu^*$. The forecaster only believes that $\mu^* \in [\mu - c, \mu+c] \subset [0, 1]$. The forecaster should be willing to accept any bet 
with non-negative expected return 
under \textit{every} $\mu^* \in [\mu - c, \mu + c]$. For example, assume the forecaster believes that a coin 
comes up heads with at least 40\% chance and at most 60\% chance.
The forecaster should be willing to win \$6 for heads, and lose \$4 for tails; similarly the forecaster should be willing to lose \$4 for heads, and win \$6 for tails. 

More generally, according to Lemma~\ref{lemma:probability_vs_bet} (proved in Appendix~\ref{appendix:proof}), a forecaster believes that
the probability of success $\Pr[Y=1] = \mu^*$
of the binary event $Y$ satisfies 
$\mu^* \in [\mu - c, \mu+c]$ if and only if she is willing to accept bets where she loses $b(Y - \mu) - |b|c, \forall b \in \mathbb{R}$. 

\begin{restatable}{lemma}{probabilitybet}
\label{lemma:probability_vs_bet}
Let $\mu, c \in (0, 1)$ such that $[\mu-c, \mu+c] \subset [0, 1]$, then a function $f: \Yc \to \mathbb{R}$ satisfies $\forall \tilde{\mu} \in [\mu - c, \mu + c]$, $\Eb_{Y \sim \tilde{\mu}}[f(Y)] \leq 0$ if and only if for some $b \in \mathbb{R}$ and $\forall y \in \lbrace 0, 1\rbrace$, $f(y) \leq b(y - \mu) - |b|c$. 
\end{restatable}

In words, a forecaster is willing to lose $f(Y)$ if $f$ has non-positive expectation under every probability the forecaster considers possible. However, every such function $f$ are smaller (i.e. forecaster loses less) than $b(Y - \mu) - |b|c$ for some $b \in \mathbb{R}$. Therefore, we only have to consider whether a forecaster is willing to accept bets of the form $b(Y - \mu) - |b|c$. 




\section{
Decisions 
with 
Unreliable Forecasts
}





In Protocol 1, agents could make decisions 
based on the forecasted probability $\mu_t, c_t$ and the agent's loss $l_t$. For example, the agent could choose
\begin{align*}
    a_t := \arg\min_{a \in \Ac} \Eb_{Y \sim \mu_t} l_t(a, Y)  \numberthis\label{eq:bayes_decision_rule}
\end{align*} 
to minimize the expected loss under the forecasted probability. 

However, how can the agent know that this decision has low expected loss under the \textit{true probability} $\mu_t^*$? This can be achieved with two desiderata, which we formalize below:

We denote the agent's maximum / average / minimum expected loss under the forecasted probability as
\begin{align*}
    L_t^{\mathrm{max}} &= \max_{\tilde{\mu} \in \mu_t \pm c_t} \Eb_{Y \sim \tilde{\mu}}[l_t(a_t, Y)] \\
    L_t^{\mathrm{avg}} &= \Eb_{Y \sim \mu_t}[l_t(a_t, Y)] \\
    L_t^{\mathrm{min}} &= \min_{\tilde{\mu} \in \mu_t \pm c_t} \Eb_{Y \sim \tilde{\mu}}[l_t(a_t, Y)]
\end{align*}
and true expected loss as 
$ L_t^* = \Eb_{Y \sim \mu^*_t}[l_t(a_t, Y)] $. 
If the agent knows that 

\textbf{Desideratum 1}  $L_t^* \in [L_t^{\mathrm{min}}, L_t^{\mathrm{max}}]$ \\
\textbf{Desideratum 2}  The interval size $c_t$ is close to $0$. 

then the agent can infer that the true expected loss $L_t^*$ is not too far off from the forecasted expected loss $L_t^{\mathrm{avg}}$. This is because if $c_t$ is small then $L^{\mathrm{min}}_t$ will be close to $L_t^{\mathrm{max}}$. Both $L_t^*$ and $L_t^{\mathrm{avg}}$ will be sandwiched in the small interval $[L_t^{\mathrm{min}}, L_t^{\mathrm{max}}]$. 

However, we show that desiderata 1 and 2 often cannot be achieved simultaneously. To guarantee $L_t^* \in [L_t^{\mathrm{min}}, L_t^{\mathrm{max}}]$ the forecaster in general must 
output individually correct probabilities (i.e. $\mu_t^* \in [\mu_t - c_t, \mu_t + c_t]$), as shown by the following  proposition (proof in Appendix~\ref{appendix:proof}). 
\begin{restatable}{prop}{losseqprob}
\label{prop:loss_eq_prob}
For any $\mu_t, c_t, \mu_t^* \in (0, 1)$ where $(\mu_t - c_t, \mu_t + c_t) \subset (0, 1)$

1. If $\mu_t^* \in [\mu_t - c_t, \mu_t + c_t]$ then $\forall l_t: \Yc \times \Ac \to \mathbb{R}$ we have $L_t^* \in [L_t^{\mathrm{min}}, L_t^{\mathrm{max}}]$

2. If $\mu_t^* \not\in [\mu_t - c_t, \mu_t + c_t]$ then $\forall l_t: \Yc \times \Ac \to \mathbb{R}$, if $\forall a \in \Ac, \ell_t(a, 0) \neq \ell_t(a, 1)$, then $L_t^* \not\in [L_t^{\mathrm{min}}, L_t^{\mathrm{max}}]$
\end{restatable}



In words, if $\mu_t^* \not\in [\mu_t - c_t, \mu_t + c_t]$, we cannot guarantee that $L_t^* \in [L_t^{\mathrm{min}}, L_t^{\mathrm{max}}]$ 
unless the agent's loss function is trivial (e.g. it is a constant function). However, in Section 2.2 we argued that it is usually impossible to achieve $\mu_t^* \in [\mu_t - c_t, \mu_t + c_t]$ unless $c_t$ is very large (i.e. $[\mu_t - c_t, \mu_t + c_t]=[0, 1]$). If $c_t$ is too large, the interval $[L_t^{\mathrm{min}}, L_t^{\mathrm{max}}]$ will be large, and the guarantee that $L_t^* \in [L_t^{\mathrm{min}}, L_t^{\mathrm{max}}]$ would be practically useless even if it were true. 
This means the forecaster cannot convey confidence in individual predictions it makes, and as a result the agent can't be very confident about the expected loss it will incur.



\subsection{Insuring against unreliable forecasts}


Since it is difficult 
to satisfy desiderata 1 and 2 simultaneously, we consider relaxing desideratum 1. In particular, we study what guarantees are possible for each individual decision maker even when $\mu_t^* \not\in [\mu_t - c_t, \mu_t + c_t]$, i.e., the prediction is wrong. 

We consider the setup where each agent can receive some side payment (a form of "insurance" which could depend on the outcome $Y$, and could be positive or negative) from the forecaster, and we would like to guarantee

\textbf{Desideratum 1'}
\begin{align*}
    \underbrace{
    L_t^* - \Eb_{Y \sim \mu_t^*} [\mathrm{payment}(Y)]}_{\textrm{\tiny True expected loss w. side payment}} \in 
    \underbrace{
    [L_t^{\mathrm{min}}, L_t^{\mathrm{max}}]}_{\textrm{\tiny Forecasted expected loss range}}
\end{align*}

In other words, we would like the expected loss 
under the true distribution to be predictable \emph{once we incorporate the side payment}.
%


Note that desideratum 1' can be trivially satisfied if the forecaster is willing to pay any side payment to the decision agent. For example, an agent can choose  $\mathrm{payment}(Y) := \Eb_{Y \sim \mu_t}[l_t(a_t, Y)] - l_t(a_t, Y)$ to satisfy desideratum 1'. However, if the forecaster offers any side payment, it could be subject to exploitation. For example, decision agents could request the forecaster to pay \$1 under any outcome $y_t$. Such a mechanism cannot be sustainable for the forecaster. 

\subsection{Insuring with fair bets}


Even though the forecaster cannot offer arbitrary payments to the decision agent, we show that the forecaster can offer a sufficiently large set of payments, such that [i] each decision agent can select a payment to satisfy Desideratum 1' and [ii] the forecaster has an algorithm to guarantee vanishing loss in the long run, even when the decision agents tries to exploit the forecaster. 

In fact, the ``fair bets'' in Section 2.3 satisfy our requirement. Specifically, the forecaster can offer the set 
\begin{align*}
    \lbrace \mathrm{payment}(Y) := b(Y - \mu_t) - |b|c_t, \forall b \in [-M, M] \rbrace
\end{align*} 
as available side payment options. The constant $M \in \mathbb{R}_+$ caps the maximum payment each decision agent can request (in our setup $l_t$ is also upper bounded by $M$).  
This set of payments satisfy both [i] (which we show in this section) and [ii] (which we show in the next section). 

Before we proceed to show [i] and [ii], for convenience, we formally write down the new protocol. Compared to Protocol 1, the decision agent selects some ``stake'' $b_t \in [-M, M]$, and receive side payment $b_t(Y - \mu_t) - |b_t|c_t$ from the forecaster.

\paragraph{Protocol 2: Decision Making with Bets} For $t = 1, \cdots, T$
\begin{enumerate}[topsep=0pt,itemsep=0.5ex,partopsep=0ex,parsep=0ex]
    \item Nature reveals observation $x_t \in \Xc$ and chooses $\mu_t^* \in [0, 1]$ without revealing it
    \item Forecaster reveals $\mu_t, c_t \in (0, 1)$ where $(\mu_t-c_t, \mu_t+c_t) \subset (0, 1)$
    \item Agent $t$ has loss function $l_t: \Ac \times \Yc \to \mathbb{R}$ and reveals action  $a_t \in \Ac$ and stake $b_t \in [-M, M]$ selected according to $\mu_t, c_t$ and $l_t$
    \item Nature samples $y_t \sim \mathrm{Bernoulli}(\mu_t^*)$ and reveals $y_t$
    \item 
    Agent incurs loss $l_t(a_t, y_t) - b_t(y_t - \mu_t) + |b_t|c_t$; forecaster incurs loss $b_t(y_t - \mu_t) - |b_t|c_t$
\end{enumerate}



Denote the agent's true expected loss with side payment as (i.e. the LHS in Desideratum 1')
\begin{align*}
    L^{\mathrm{pay}}_t &:= \underbrace{L_t^*}_{\text{decision loss}} - \underbrace{\Eb_{Y \sim \mu_t^*}[b_t(Y - \mu_t) + |b_t|c_t]}_{\text{payment from forecaster}} \numberthis\label{eq:guarantee_with_payment2}
\end{align*}
then we have the following guarantee\footnote{For the more general version of the proposition in the multi-class setup, see Appendix~\ref{appendix:multiclass}.} for any choice of $\mu_t, c_t, \mu_t^*, a_t$ and $l_t$
\begin{restatable}{prop}{decisionguarantee}
\label{prop:decision_guarantee}
If the stake $b_t = l_t(a_t, 1) - l_t(a_t, 0)$ then 
$
    L^{\mathrm{pay}}_t \in [L_t^{\mathrm{min}}, L_t^{\mathrm{max}}]
$
\end{restatable}

In words, the agent has a choice of 
stake $b_t$ that only depends on variables known to the agent ($l_t$ and $a_t$) and does not depend on variables unknown to the agent ($\mu_t^*$, $y_t$). If the agent chooses this $b_t$, she can be certain that desideratum 1' is satisfied, 
regardless of what the forecaster or nature does (they can choose any $\mu_t, c_t, \mu_t^*$). 

This mechanism allows the agent to \textbf{make decisions as if the forecasted probability is correct}, i.e. as if $\mu_t^* \in [\mu_t - c_t, \mu_t + c_t]$. This is because Proposition~\ref{prop:decision_guarantee} is true for any choice of action $a_t$ (as long as the agent chooses $b_t$ according to Proposition~\ref{prop:decision_guarantee} after selecting $a_t$). Intuitively, for any action $a_t$ the agent selects, she can guarantee to achieve a total loss close to $\Eb_{Y \sim \mu_t} l_t(a_t, Y)$ (assuming $c_t$ is small). This is the same guarantee she would get as if $\mu_t^* \in [\mu_t - c_t, \mu_t + c_t]$. 

In addition, if [ii] is satisfied (i.e. the forecaster has vanishing loss), the forecaster also doesn't lose anything, so should have no incentive to avoid offering these payments.
We discuss this in the next section.

\begin{algorithm}
\caption{Post-Processing for Exactness}
\label{alg:online_prediction}
\SetAlgoLined
\DontPrintSemicolon
Invoke Algorithm~\ref{alg:online_prediction_original} and \ref{alg:swap_minimization} with $K = (T/\log T)^{1/4}$ \;
\For {$t=1, \cdots, T$} {
    Receive $\hat{\mu}_t$ and $\hat{c}_t$ from Algorithm~\ref{alg:online_prediction_original} \;
    {\color{cyan} Receive $\lambda_t$} from Algorithm~\ref{alg:swap_minimization} \;
    Output $\mu_t = \hat{\mu}_t$, $c_t = \hat{c}_t + \lambda_t$ \; 
    Input $y_t$ and $b_t$ \;
    Set $r_t = (b_t / \sqrt{|b_t|}) (\mu_t - y_t) - \sqrt{|b_t|} \hc_t $, $s_t = -\sqrt{|b_t|}$,
    {\color{blue} Send $(r_t, s_t)$ } to Algorithm~\ref{alg:swap_minimization}\; 
}

\end{algorithm}
\vspace{-3mm}

\begin{algorithm}
\caption{Online Prediction}
\label{alg:online_prediction_original}
\SetAlgoLined
\DontPrintSemicolon
Choose any initial value for $\theta_1, \phi_1$\;
\For {$t=1, \cdots, T$} {
    Input $x_t$ and output $\hat{\mu}_t = \mu_{\theta_t}(x_t)$, $\hat{c}_t = c_{\phi_t}(x_t)$ \; 
    Input $y_t$ and $b_t$ \;
    $\theta_{t+1} = \theta_t - \eta \frac{\partial}{\partial \theta} (\mu_{\theta_t}(x_t) - y_t)^2$\;
    $\phi_{t+1} = \phi_t - \eta \frac{\partial}{\partial \phi} \left(b_t (\hat{\mu}_t - y_t) - |b_t| c_{\phi_t}(x_t) \right)^2$\;
}
\end{algorithm}

\section{Probability Forecaster Strategy} 

In this section we study the forecaster's strategy. 
As motivated in the previous section, the goal of the forecaster (in Protocol 2) is to:

\textbf{1)} have non-positive cumulative loss when $T$ is large, so that the side payments are sustainable \\
\textbf{2)} output the smallest $c_t$ 
compatible with 1), so that forecasts are as sharp as possible

Specifically, the forecaster's average cumulative loss (up to time $T$) in Protocol 2 is 
\begin{align*}
    \frac{1}{T} \sum_{t=1}^T b_t (\mu_t - y_t) - |b_t| c_t  \numberthis\label{eq:forecaster_loss}
\end{align*}
Whether Eq.(\ref{eq:forecaster_loss}) is non-positive or not 
depends on the actions of all the players: forecaster $\mu_t, c_t$, nature $y_t$ and agent $b_t$. Our focus is on the forecaster, so 
we say that a sequence of forecasts $\mu_t, c_t, t=1, 2, \cdots$ is \textbf{asymptotically sound} relative to $y_1, b_1, y_2, b_2, \cdots$ if the forecaster loss in Protocol 2 is non-positive, i.e. 
\begin{align*}
    \lim\sup_{T \to \infty} \frac{1}{T} \sum_{t=1}^T b_t (\mu_t - y_t) - |b_t| c_t \leq 0 \numberthis\label{eq:soundness}
\end{align*}


In subsequent development we will use a stronger definition than Eq.(\ref{eq:soundness}). We say that a sequence of forecasts $\mu_t, c_t, t=1, 2, \cdots$ is
\textbf{asymptotically exact} relative to $y_1, b_1, y_2, b_2, \cdots$ if the forecaster loss in Protocol 2 is exactly zero, i.e.
\begin{align*}
    \lim\sup_{T \to \infty} \frac{1}{T} \sum_{t=1}^T b_t (\mu_t - y_t) - |b_t| c_t = 0 \numberthis\label{eq:exactness}
\end{align*}
Intuitively asymptotic soundness requires that the forecaster should not lose in the long run; asymptotic exactness requires that the forecaster should neither lose nor win in the long run --- a stronger requirement.\footnote{In mechanism design literature, Eq.(\ref{eq:soundness}) and Eq.(\ref{eq:exactness}) are typically referred to as weak and strong budget balanced. Here we use the terminology in probability forecasting literature.}

The reason we focus on asymptotic exactness is because 
the forecaster should output the smallest possible $c_t$ 
to achieve sharp forecasts.
Observe that the left hand side of Eq.(\ref{eq:soundness}) is increasing if $c_t$ decreases. 
Therefore, whenever the forecaster is asymptotically sound but not asymptotically exact (i.e. the left hand side in Eq.(\ref{eq:soundness}) is strictly negative), there is some room to decrease $c_t$ without violating asymptotic soundness. 



\begin{algorithm}
\caption{Swap Regret Minimization}
\label{alg:swap_minimization}
\SetAlgoLined
\DontPrintSemicolon
Input: number of discrete interval $K$\;
Partition $[-1, 1]$ into equal intervals $[-1=v_0, v_1)$, $\cdots$, $[v_{K-1}, v_K=1]$ \;
For each interval init an empty set $\Dc_k$, set $v^0 = 0$ \;
\For {$t=1, \cdots, T$} {
    Initialize an empty ordered list $\Vc^t$\;
    Initialize $v^t = v^{t-1}$ and \While {$v^t \not\in \Vc^t$}{
        $\lambda_t^{v^t} = \underset{\lambda \in [-1, 1)}{\arg\inf} \frac{1}{|\Dc_{v^t}|} \sum_{r_t, s_t \in \Dc_{v^t}} (r_t + s_t \lambda)^2$\; 
        Append $v^t$ to $\Vc^t$\;
        Set $v^t$ as the $k$ that satisfies $\lambda_t^{v^t} \in [v_k, v_{k+1})$\;
    }
    Remove all elements before $v^t$ from $\Vc^t$ \;
    Select $v^t$ uniform randomly from $\Vc^t$ \;
    Choose $\lambda_t = \lambda_t^{v^t}$ and {\color{cyan} send $\lambda_t$} to Algorithm~\ref{alg:online_prediction} \;
    {\color{blue} Receive $(r_t, s_t)$} from Algorithm~\ref{alg:online_prediction}, add to $\Dc_{v^t}$ 
}
\end{algorithm}

\subsection{Online Forecasting Algorithm}

We aim to achieve asymptotic exactness with minimal assumptions on $y_t, b_t, t=1, 2, \cdots$ (we only assume boundedness). This is challenging for two reasons: an adversary could select $y_t, b_t, t=1, 2, \cdots$ to violate asymptotic exactness as much as possible (e.g. decision agents could try to profit on the forecaster's loss); in Protocol 2 the agent's action $b_t$ is selected \textit{after} the forecaster's prediction $\mu_t, c_t$ are revealed, so the agent has last-move advantage. 

Nevertheless asymptotic exactness can be achieved as shown in Theorem~\ref{thm:online_sound_forecast} (proof in Appendix~\ref{appendix:swap}). In fact, we design a post-processing algorithm that modifies the prediction of a base algorithm (similar to recalibration~\cite{kuleshov2017estimating,guo2017calibration}). Algorithm~\ref{alg:online_prediction} can modify any base algorithm (as long as the base algorithm outputs some $\mu_t, c_t$ at every time step) to achieve asymptotic exactness, even though the finite time performance could be hurt by a poor base prediction algorithm.  

\begin{restatable}{theorem}{onlinesoundness}
\label{thm:online_sound_forecast}
Suppose there is a constant $M > 0$ such that $\forall t$, $|b_t| \leq M$, there exists an algorithm to output $\mu_t, c_t$ in Protocol 2 that is asymptotically exact for $\mu_t^*, b_t$ generated by any strategy of nature and agent. In particular, Algorithm~\ref{alg:online_prediction} satisfies 
\begin{align*}
    \left( \frac{1}{T} \sum_{t=1}^T b_t (\mu_t - y_t) - |b_t| c_t \right)^2 = O\left(\sqrt{\frac{\log T}{T}} \right) \numberthis\label{eq:online_sound_forecast}
\end{align*} 
\end{restatable}

For this paper we use as our base algorithm a simple online gradient descent algorithm~\cite{zinkevich2003online} shown in Algorithm~\ref{alg:online_prediction_original}. Specifically Algorithm~\ref{alg:online_prediction_original} learns two regression models (such as neural networks with a single real number as output) $\mu_\theta$ and $c_\phi$.
$\mu_\theta$ is trained to predict $\mu_t^*$ by minimizing the standard $L_2$ loss 
$
    \min_\theta \sum_{\tau=1}^t (\mu_\theta(x_\tau) - y_\tau)^2 
$
while $c_\phi$ is trained to to minimize the squared payoff of each bet
$
\min_\phi \sum_{\tau=1}^t (b_\tau (\hat{\mu}_\tau - y_\tau) - |b_\tau| c_\phi(x_\tau))^2  
$



Based on Algorithm~\ref{alg:online_prediction_original}, Algorithm~\ref{alg:online_prediction} learns an additional ``correction'' parameter $\lambda_t \in \mathbb{R}$ by invoking Algorithm~\ref{alg:swap_minimization}. 
Intuitively, up to time $t$, if the forecaster has positive cumulative loss in Protocol 2, then the $c_t$s have been too small in the past, Algorithm~\ref{alg:online_prediction} will select a larger $\lambda_t$ to increase $c_t$; conversely if the forecaster has negative cumulative loss, then the $c_t$s have been too large in the past, and Algorithm~\ref{alg:online_prediction} will select a smaller $\lambda_t$ to decrease $c_t$. 

Despite the straight-forward intuition, the difficulty comes from ensuring Theorem~\ref{thm:online_sound_forecast} for \textit{any} sequence of $y_t, b_t, t=1, \cdots$. 
In fact, Algorithm~\ref{alg:swap_minimization} needs to be a swap regret minimization algorithm~\cite{blum2007external}. Appendix~\ref{appendix:swap} provides a detailed explanation and proof. 

\paragraph{Special Case: Monotonic Loss}
In general, the forecaster selects $c_t$ carefully to achieve asymptotic exactness and protect itself from exploitation. However, there are special cases where the $c_t$ is not necessary (i.e. the forecaster can always output $c_t \equiv 0$). 

In particular, $c_t$ is not necessary whenever the loss function satisfies $\forall t, l_t(1, a_t) \geq l_t(0, a_t)$. Intuitively, $y_t=1$ is never better (incurs equal or higher loss) than $y_t=0$. 
For example, ineffective treatment is never better than effective treatment; delayed flight is never better than on-time flights. Under this assumption and according to Proposition~\ref{prop:decision_guarantee}, decision makers can choose a non-negative stake $0 \leq b_t := l_t(1, a_t) - l_t(0, a_t)$ to ensure $
    L^{\mathrm{pay}}_t \in [L_t^{\mathrm{min}}, L_t^{\mathrm{max}}]
$. In other words, in Protocol 2 we can restrict $b_t \geq 0$ without losing the guarantee of Proposition~\ref{prop:decision_guarantee}. In this situation 
the forecaster can achieve asymptotic exactness even when $c_t \equiv 0$
\begin{restatable}{corollary}{onlinesoundnesstwo}
\label{cor:online_sound_forecast}
Under the assumptions of Theorem~\ref{thm:online_sound_forecast} if additionally $b_t \geq 0, \forall t$, there exists an algorithm to output $\mu_t$ in Protocol 2 with $c_t \equiv 0$ that is asymptotically exact for $\mu_t^*, b_t$ generated by any strategy of nature and agent. 
\end{restatable}

\subsection{Offline Forecasting}

Our new definition of asymptotic soundness in Eq.(\ref{eq:soundness}) can be extended to the offline setup, where nature's action in Protocol 2 is i.i.d. sampled from random variables $X, Y$, i.e.
$
    x_t \sim X, \mu_t^* = \Eb[Y \mid x_t]
$.
In addition, the agent's action $b_t$ in Protocol 2 is a (fixed) function of $x_t$ i.e. $b_t = b(x_t)$ for some $b: \Xc \to [-M, M]$. In this setup, the forecaster can also select its actions $\mu_t, c_t$ based on fixed functions of $x_t$. 


In Appendix~\ref{appendix:offline} we formally define soundness in the offline setup, and show that for certain choices of agent's action $b$ we can recover 
existing notions of calibration \cite{dawid1985calibration,guo2017calibration,kleinberg2016inherent,kumar2019verified} or multicalibration~\cite{hebert2017calibration}.  
In other words, 
if a forecaster satisfies the existing notions of calibration, there are some functions $b: \Xc \to [-M, M]$: as long as the decision making agents limit their actions to $b_t = b(x_t)$, the forecaster will be asymptotically sound. 
The benefit is that once deployed, the forecaster does not have to be updated (compared to the online setup where the forecaster must continually update via Algorithm~\ref{alg:online_prediction}). 
However, the short-coming is that we make strong assumptions on how the agents choose bets to insure themselves.

\section{Related Work} 
\textbf{Calibration}: A forecaster is calibrated if among the times the forecaster predicts that an event happens with $\alpha$ probability, the event indeed happens $\alpha$ fraction of the times~\cite{brier1950verification, murphy1973new, dawid1984present,platt1999probabilistic, zadrozny2001obtaining, guo2017calibration}. Calibration can be achieved even when the data is not i.i.d.~\cite{cesa2006prediction,kuleshov2016estimating}. However, calibration is an average performance measurement and provides no guarantee on the correctness of individual probability predictions. 

\textbf{Scoring rule}: a (proper) scoring rule is a function $s(y, p_Y)$ that measures the ``quality'' of a probability forecast $p_Y$ if the outcome $y$ is observed~\cite{brier1950verification,savage1971elicitation,gneiting2007strictly,dawid2014theory}. 
However, achieving a high score only reflects high \textit{average} quality, rather than the quality of individual predictions. 

\textbf{Conformal prediction}: Many applications only require a confidence interval (i.e. a subset of $\Yc)$ instead of the joint probability. 
A confidence interval forecaster (or conformal forecaster) is $\delta$-exact if $1-\delta$ proportion of the predicted intervals contain the observed outcome. There are algorithms that guarantee exactness for exchangeable data~\cite{vovk2005algorithmic,shafer2008tutorial}. However, exactness guarantees the proportion of predictions that contain the observed outcome, rather than any individual prediction. 

The above approaches provide no guarantees on the correctness of individual predictions. The classic method that can guarantee individual predictions is \textbf{non-parametric learning}. Algorithms such as nearest neighbor or Gaussian processes can produce correct individual probabilities with infinite training data~\cite{bishop2006pattern}, but have no guarantees when
training data is finite or not i.i.d.  

In the finite data regime, a notable research direction is \textbf{individual calibration}, i.e. calibration on 
every data sample. Individual calibration is sometimes possible with a randomized forecaster~\cite{zhao2020individual}. However, for randomized forecasts, calibration cannot be interpreted as forecasting correct probabilities. Without randomization, individual calibration is often impossible~\cite{vovk2005algorithmic,vovk2012conditional,zhao2020individual,barber2019limits}. 

Individual calibration can be relaxed to \textbf{group calibration}, i.e. calibration on  pre-specified subsets of the data~\cite{kleinberg2016inherent}. Notably, \cite{hebert2017calibration} achieve calibration for a parametric set of subsets. Several impossibility results~\cite{kleinberg2016inherent,pleiss2017fairness} show that often group calibration cannot be meaningfully achieved.


\textbf{Our contribution} Our approach has the main desired effect of individual calibration
(decision makers can confidently use the forecasted probability "as if" it is correct) 
without actually achieving individual calibration, hence are not limited by the impossibility results above. 
The key difference that makes our guarantees possible (without i.i.d. assumptions, well specification assumptions, infinite data, or randomization) is that the forecasts depend on downstream decision tasks. Rather than predicting perfect probabilities, we aim for the attainable objective of achieving exactness for actually encountered decision tasks.

\begin{figure*}[h]
    \centering
    \includegraphics[width=1.0\linewidth]{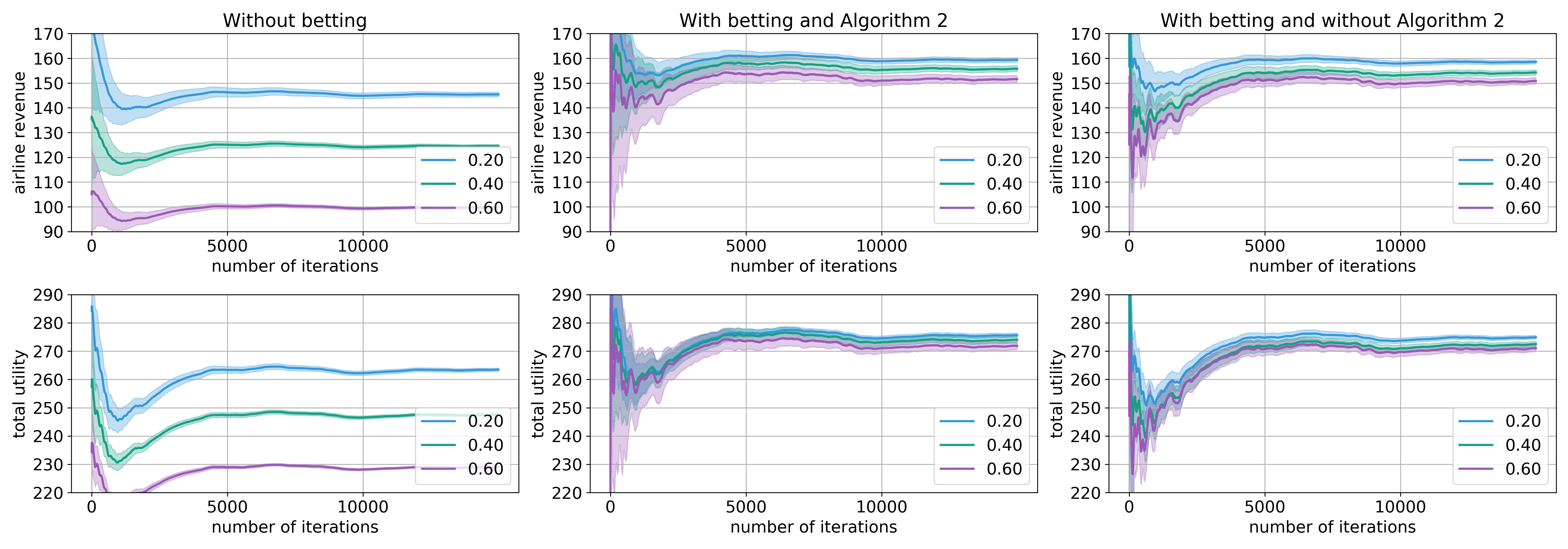}
    \vspace{-3mm}
    \caption{The airline's revenue (\textbf{Top}) and total utility (of both airline and passenger, \textbf{Bottom}) with and without the betting mechanism. Different colors represent the percentage of cautious passengers. The x-axis represents the number of flights that has happened, and the y-axis represents the average utility per passenger across all past flights. \textbf{Left}: Without the betting mechanism that insure passengers against delay \textbf{Middle and Right}: With the betting mechanism, the airline revenue increases (because it is able to charge a higher ticket price due to increased demand) and the total utility increases. The middle panel is the utility with both Algorithm~\ref{alg:online_prediction} and Algorithm~\ref{alg:swap_minimization}, while the right panel only uses Algorithm~\ref{alg:online_prediction} (i.e. it always sets $\lambda_t=0$). In general the middle panel achieves faster convergence, so with fewer iterations, the utility is better than the right panel. 
    }
    \label{fig:airline}
    \vspace{-3mm}
\end{figure*}

\section{Case Study on Flight Delays}

In this section we study a practical application that could benefit from our proposed mechanism. Compared to other means of transport, flights are often the fastest, but usually the least punctual. Different passengers may have different losses in case of delay. For example, if a passenger needs to attend an important event on-time, the loss from a delay can be very large, and the passenger might want to choose an alternative transportation method. The airline company could predict the probability of delay, and each passenger could use the probability to compute their expected loss before deciding to fly or not. However, as argued in Section~\ref{sec:impossibility}, there is in general no good way to know that these probabilities are correct. 
Even worse, the airline may have the incentive to under-report the probability of delay to attract passengers. 

Instead the airline can use Protocol 2 to convey confidence to the passengers that the delay probability is accurate. In this case, Protocol 2 has a simple form that can be easily explained to passengers as a ``delay insurance''. In particular, if a passenger buys a ticket, he can choose to insure himself against delay by specifying the amount $b_t^1$ he would like to get paid if the airplane is delayed. The airlines provides a quote on the cost $b_t^0$ of the insurance (i.e. the passenger pays $b_t^0$ if the flight is not delayed). 
Note that this would be equivalent to Protocol 2 if the airline first predicts the probability of delay $\mu_t, c_t$ and then quotes
$
    b_t^0 :=  \frac{b_t^1(\mu_t + c_t)}{1 - \mu_t - c_t}
$. 

If a passenger buys the right insurance according to Proposition~\ref{prop:decision_guarantee}, their expected utility (or negative loss) will be fixed --- she does not need to worry that the predicted delay probability might be incorrect. In addition, if the airline follows Algorithm~\ref{alg:online_prediction} the airline is also guaranteed to not lose money from the ``delay insurance'' in the long run (no matter what the passengers do), so the airline should be incentivized to implement the insurance mechanism to benefit its passengers ``for free''. 


\paragraph{Passenger Model} 
Since the passengers' utility functions are unknown, we model three types of passengers that differ by their assumptions on $\mu_t^*$ when they make their decision:
\begin{enumerate}[topsep=0pt,itemsep=0ex,partopsep=1ex,parsep=1ex]
    \item Naive passengers don't care about delays and assume the airline doesn't delay.
    \item Trustful passengers assume the delay probability forecasted by the airline is correct. 
    \item Cautious passengers assume the worst (i.e. they choose actions that maximizes their worst case utility)  
\end{enumerate}

In this experiment we will vary the proportion of cautious passengers, and equally split the remaining passengers between naive and trustful. The naive and trustful passengers do not care about the risk of mis-prediction, so they do not buy the delay insurance (i.e. they always choose $b_t^1=0$), while cautious passengers always buy insurance that maximize worst case utility.

\subsection{Simulation Setup}

\paragraph{Dataset} We use the flight delay and cancellation dataset~\cite{kaggle2015airline} from the year 2015, and use flight records of the single biggest airline (WN). As input feature, we convert the source airport, target airport, and scheduled time into one-hot vectors, and binarize the arrival delay into 1 (delay $>$ 20min) and 0 (delay $<$ 20min). \blfootnote{Code is available at \\ https://github.com/ShengjiaZhao/ForecastingWithBets}
%
We use a two layer neural network with the leaky ReLU activation for prediction. 

\paragraph{Passenger Utility} Let $y \in \lbrace 0, 1 \rbrace$ denote whether a delay happens, and $a \in \lbrace 0, 1 \rbrace$ denote whether the passenger chooses to ride the plane. We model the passenger utility (negative loss) as
\begin{align*}
    -l(y, a) = \left\lbrace \begin{array}{ll} 
    y=*, a=0 & r^{\mathrm{alt}} \\
    y=0, a=1 & r^{\mathrm{trip}} - c^{\mathrm{ticket}}  \\
    y=1, a=1 & r^{\mathrm{trip}} - c^{\mathrm{ticket}} - c^{\mathrm{delay}}
    \end{array} \right.  
\end{align*}
where $r^{\mathrm{alt}}$ is the utility of the alternative option (e.g. taking another transportation or cancelling the trip). For simplicity we assume that this is a single real number. $r^{\mathrm{trip}}$ is the reward of the trip, $c^{\mathrm{ticket}}$ is the cost of the ticket, and $c^{\mathrm{delay}}$ is the cost of a delayed flight. 
For each flight we sample 1000 potential passengers by randomly drawing the values $r^{\mathrm{alt}}, r^\mathrm{trip}$ and $c^{\mathrm{delay}}$ (for details see appendix).

\paragraph{Airline Pricing} 
Based on the passenger type (naive, trustful, cautious) and passenger parameter $r^{\mathrm{alt}}, r^\mathrm{trip}$ and $c^{\mathrm{delay}}$, each passenger will have a maximum they are willing to pay for the flight. For simplicity we assume the airline will choose $c^{\mathrm{ticket}}$ at the highest price for which it can sell 300 tickets. The passengers who are willing to pay more than $c^{\mathrm{ticket}}$ will choose $a=1$, and other passengers will choose $a=0$.

\begin{figure*}[ht]
    \centering
    \includegraphics[width=0.85\linewidth]{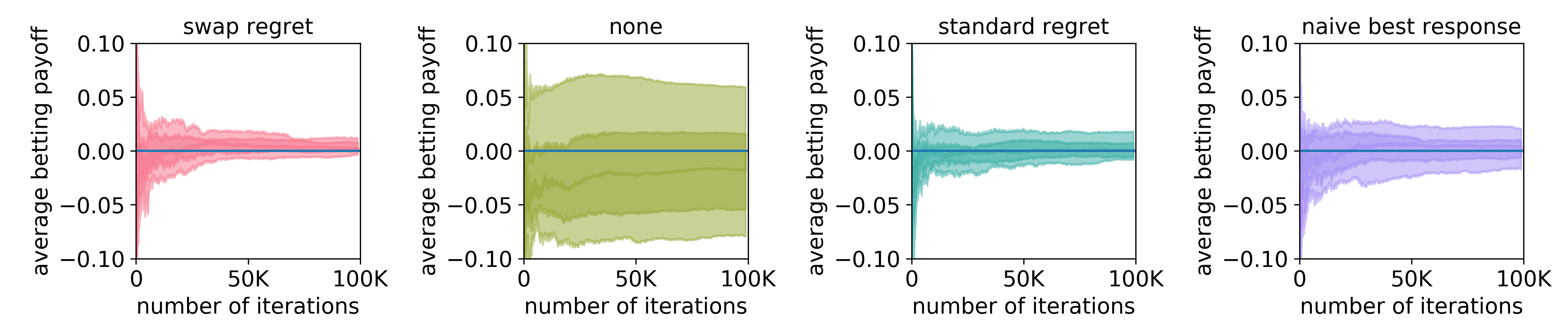}
    \vspace{-3mm} 
    \caption{Comparing forecaster loss in Protocol 2 for different forecaster algorithms on MNIST (results for Adult dataset are in appendix~\ref{appendix:additional_experiments}). Each plot is an average performance across 20 different decision tasks, where we plot the top 10\%, 25\%, 50\%, 75\%, 90\% quantile in forecaster loss. If the forecaster achieves asymptotic exactness defined in Eq.(\ref{eq:exactness}), then the loss should be close to $0$. \textbf{Left} panel is Algorithm~\ref{alg:online_prediction}, and the rest are other seemingly reasonable algorithms explained in Section~\ref{sec:additional_experiment_setup}. The loss of a forecaster that use Algorithm~\ref{alg:online_prediction} typically converges to $0$ faster, while alternative algorithms often fail to converge.}  
    \vspace{-3mm}
    \label{fig:betting_payoff}
\end{figure*}

\begin{figure*}[ht]
    \centering
    \includegraphics[width=0.85\linewidth]{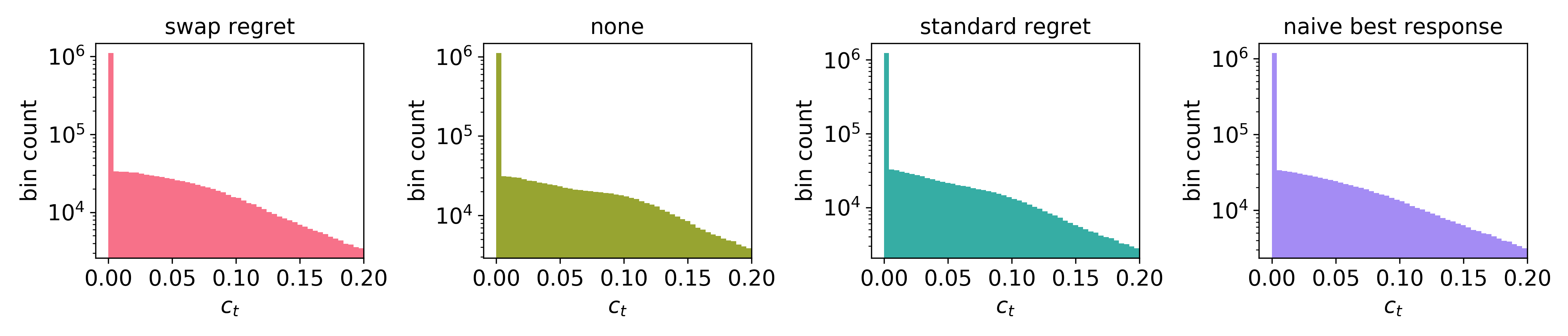} 
    \vspace{-3mm}
    \caption{Histogram of the interval size $c_t$ produced by the forecaster algorithm across all the tasks. There is no noticeable difference between the different algorithms. Notably the interval sizes are typically quite small, and big interval size is exponentially less common.}
    \label{fig:betting_cost}
    \vspace{-3mm}
\end{figure*}

\subsection{Delay Insurance Improves Total Utility}

The simulation results are shown in Figure~\ref{fig:airline}. Using the betting mechanism is strictly better for both the airline's revenue (i.e. ticket price * number of tickets) and the total utility (airline revenue + passenger utility). This is because the cautious passengers always make decisions to maximize their worst case utility. With the betting mechanism, their worst case utility becomes closer to their actual true utility, so their decision ($a=1$ or $a=0$) will better maximize their true utility. The airline also benefits because it can charge a higher ticket price due to increased demand (more cautious passengers will choose $a=1$). 

We also consider several alternatives to Algorithm~\ref{alg:swap_minimization}. The alternative algorithms do not provide theoretical guarantees; in practice, they also achieve worse convergence to the final utility. This is be a reason to prefer Algorithm~\ref{alg:swap_minimization} if the number of iterations $T$ is small. 





\section{Additional Experiments}
\label{sec:additional_experiment}
\label{sec:additional_experiment_setup}

We further verify our theory with simulations a diverse benchmark of decision tasks. We also do ablation study to show that Algorithm~\ref{alg:swap_minimization} is necessary. Several simpler alternatives often fail to achieve asymptotic exactness and have worse empirical performance.

\paragraph{Dataset and Decision Tasks} We use the MNIST and UCI Adult~\cite{Dua:2019} datasets. MNIST is a multi-class classification dataset; we convert it to binary classification by choosing
$
    \Pr[Y = 1 \mid l(x) = i] = (i+1) / 11
$
where the $l(x) \in \lbrace 0, 1, \cdots, 9 \rbrace$ is the digit category. We also generate a benchmark consisting of 20 different decision tasks. For details see Appendix~\ref{appendix:additional_experiments}.




\paragraph{Comparison} We compare several forecaster algorithms that differ in whether they use Algorithm~\ref{alg:swap_minimization} to adjust the parameter $\lambda_t$. In particular, \textbf{swap regret} refers to Algorithm~\ref{alg:swap_minimization}; \textbf{none} does not use $\lambda_t$ and simply set it to $0$; \textbf{standard regret} minimizes the standard regret rather than the swap regret; \textbf{naive best response} chooses the $\lambda_t$ that would have been optimal were it counter-factually applied to the past iterations. 

\paragraph{Forecaster Model} As in the previous experiment, we use a two layer neural network as the forecaster $\mu_\theta$ and $c_\phi$. For the results shown in Figure~\ref{fig:true_utility} we also use histogram binning~\cite{naeini2015obtaining} on the entire validation set to recalibrate $\mu_\theta$, such that $\mu_\theta$ satisfies standard calibration~\cite{guo2017calibration}.

\paragraph{Results} The results are plotted in Figure~\ref{fig:betting_payoff},\ref{fig:betting_cost} in the main paper and Figure~\ref{fig:betting_payoff2},\ref{fig:true_utility} in Appendix~\ref{appendix:additional_experiments}. There are three main observations: \textbf{1)} Even when a forecaster is calibrated, for individual decision makers, the expected loss under the forecaster probability is almost always incorrect. 
\textbf{2)} Algorithm~\ref{alg:online_prediction} has good empirical performance. In particular, the guarantees of Theorem~\ref{thm:online_sound_forecast} can be achieved within a reasonable number of time steps, and the interval size $c_t$ is usually small.
\textbf{3)} Seemingly reasonable alternatives to Algorithm~\ref{alg:online_prediction} often empirically fail to be asymptotically exact.

\FloatBarrier

\section{Conclusion}

In this paper, we propose an alternative solution to address the impossibility of individual calibration based on an insurance between the forecaster and decision makers. Each decision maker can make decisions as if the forecasted probability is correct, while the forecaster can also guarantee not losing in the long run. Future work can explore social issues that arise, such as honesty~\cite{foreh2003honesty}, fairness~\cite{dwork2012fairness}, and moral/legal implications.

\section{Acknowledgements} 

This research was supported by AFOSR (FA9550-19-1-0024), NSF (\#1651565, \#1522054, \#1733686), JP Morgan, ONR, FLI, SDSI, Amazon AWS, and SAIL. 

\bibliographystyle{apalike}
\bibliography{ref}

\begin{thebibliography}{}

\bibitem[Barber et~al., 2019]{barber2019limits}
Barber, R.~F., Candes, E.~J., Ramdas, A., and Tibshirani, R.~J. (2019).
\newblock The limits of distribution-free conditional predictive inference.
\newblock {\em arXiv preprint arXiv:1903.04684}.

\bibitem[Bielinski et~al., 2014]{bielinski2014preemptive}
Bielinski, S.~J., Olson, J.~E., Pathak, J., Weinshilboum, R.~M., Wang, L.,
  Lyke, K.~J., Ryu, E., Targonski, P.~V., Van~Norstrand, M.~D., Hathcock,
  M.~A., et~al. (2014).
\newblock Preemptive genotyping for personalized medicine: design of the right
  drug, right dose, right time—using genomic data to individualize treatment
  protocol.
\newblock In {\em Mayo Clinic Proceedings}, pages 25--33. Elsevier.

\bibitem[Bishop, 2006]{bishop2006pattern}
Bishop, C.~M. (2006).
\newblock {\em Pattern recognition and machine learning}.
\newblock springer.

\bibitem[Blum and Mansour, 2007]{blum2007external}
Blum, A. and Mansour, Y. (2007).
\newblock From external to internal regret.
\newblock {\em Journal of Machine Learning Research}, 8(Jun):1307--1324.

\bibitem[Brier, 1950]{brier1950verification}
Brier, G.~W. (1950).
\newblock Verification of forecasts expressed in terms of probability.
\newblock {\em Monthly weather review}, 78(1):1--3.

\bibitem[Cesa-Bianchi and Lugosi, 2006]{cesa2006prediction}
Cesa-Bianchi, N. and Lugosi, G. (2006).
\newblock {\em Prediction, learning, and games}.
\newblock Cambridge university press.

\bibitem[Dawid, 1984]{dawid1984present}
Dawid, A.~P. (1984).
\newblock Present position and potential developments: Some personal views
  statistical theory the prequential approach.
\newblock {\em Journal of the Royal Statistical Society: Series A (General)},
  147(2):278--290.

\bibitem[Dawid, 1985]{dawid1985calibration}
Dawid, A.~P. (1985).
\newblock Calibration-based empirical probability.
\newblock {\em The Annals of Statistics}, pages 1251--1274.

\bibitem[Dawid and Musio, 2014]{dawid2014theory}
Dawid, A.~P. and Musio, M. (2014).
\newblock Theory and applications of proper scoring rules.
\newblock {\em Metron}, 72(2):169--183.

\bibitem[De~Finetti, 1931]{de1931sul}
De~Finetti, B. (1931).
\newblock On the subjective meaning of probability.
\newblock {\em Fundamenta mathematicae}, 17(1):298--329.

\bibitem[DoT, 2017]{kaggle2015airline}
DoT (2017).
\newblock 2015 flight delays and cancellations.
\newblock \url{https://www.kaggle.com/usdot/flight-delays}.

\bibitem[Dua and Graff, 2017]{Dua:2019}
Dua, D. and Graff, C. (2017).
\newblock {UCI} machine learning repository.

\bibitem[Dwork et~al., 2012]{dwork2012fairness}
Dwork, C., Hardt, M., Pitassi, T., Reingold, O., and Zemel, R. (2012).
\newblock Fairness through awareness.
\newblock In {\em Proceedings of the 3rd innovations in theoretical computer
  science conference}, pages 214--226. ACM.

\bibitem[Foreh and Grier, 2003]{foreh2003honesty}
Foreh, M.~R. and Grier, S. (2003).
\newblock When is honesty the best policy? the effect of stated company intent
  on consumer skepticism.
\newblock {\em Journal of consumer psychology}, 13(3):349--356.

\bibitem[Gneiting and Raftery, 2007]{gneiting2007strictly}
Gneiting, T. and Raftery, A.~E. (2007).
\newblock Strictly proper scoring rules, prediction, and estimation.
\newblock {\em Journal of the American statistical Association},
  102(477):359--378.

\bibitem[Guo et~al., 2017]{guo2017calibration}
Guo, C., Pleiss, G., Sun, Y., and Weinberger, K.~Q. (2017).
\newblock On calibration of modern neural networks.
\newblock {\em arXiv preprint arXiv:1706.04599}.

\bibitem[Halpern, 2017]{halpern2017reasoning}
Halpern, J.~Y. (2017).
\newblock {\em Reasoning about uncertainty}.
\newblock MIT press.

\bibitem[H{\'e}bert-Johnson et~al., 2017]{hebert2017calibration}
H{\'e}bert-Johnson, U., Kim, M.~P., Reingold, O., and Rothblum, G.~N. (2017).
\newblock Calibration for the (computationally-identifiable) masses.
\newblock {\em arXiv preprint arXiv:1711.08513}.

\bibitem[Jaynes, 1996]{jaynes1996probability}
Jaynes, E.~T. (1996).
\newblock {\em Probability theory: the logic of science}.
\newblock Washington University St. Louis, MO.

\bibitem[Kleinberg et~al., 2016]{kleinberg2016inherent}
Kleinberg, J., Mullainathan, S., and Raghavan, M. (2016).
\newblock Inherent trade-offs in the fair determination of risk scores.
\newblock {\em arXiv preprint arXiv:1609.05807}.

\bibitem[Kuleshov and Ermon, 2016]{kuleshov2016estimating}
Kuleshov, V. and Ermon, S. (2016).
\newblock Estimating uncertainty online against an adversary.
\newblock {\em arXiv preprint arXiv:1607.03594}.

\bibitem[Kuleshov and Ermon, 2017]{kuleshov2017estimating}
Kuleshov, V. and Ermon, S. (2017).
\newblock Estimating uncertainty online against an adversary.
\newblock In {\em Thirty-First AAAI Conference on Artificial Intelligence}.

\bibitem[Kumar et~al., 2019]{kumar2019verified}
Kumar, A., Liang, P.~S., and Ma, T. (2019).
\newblock Verified uncertainty calibration.
\newblock In {\em Advances in Neural Information Processing Systems}, pages
  3792--3803.

\bibitem[Murphy, 1973]{murphy1973new}
Murphy, A.~H. (1973).
\newblock A new vector partition of the probability score.
\newblock {\em Journal of applied Meteorology}, 12(4):595--600.

\bibitem[Naeini et~al., 2015]{naeini2015obtaining}
Naeini, M.~P., Cooper, G.~F., and Hauskrecht, M. (2015).
\newblock Obtaining well calibrated probabilities using bayesian binning.
\newblock In {\em Proceedings of the... AAAI Conference on Artificial
  Intelligence. AAAI Conference on Artificial Intelligence}, volume 2015, page
  2901. NIH Public Access.

\bibitem[Ng et~al., 2009]{ng2009agenda}
Ng, P.~C., Murray, S.~S., Levy, S., and Venter, J.~C. (2009).
\newblock An agenda for personalized medicine.
\newblock {\em Nature}, 461(7265):724--726.

\bibitem[Platt et~al., 1999]{platt1999probabilistic}
Platt, J. et~al. (1999).
\newblock Probabilistic outputs for support vector machines and comparisons to
  regularized likelihood methods.
\newblock {\em Advances in large margin classifiers}, 10(3):61--74.

\bibitem[Pleiss et~al., 2017]{pleiss2017fairness}
Pleiss, G., Raghavan, M., Wu, F., Kleinberg, J., and Weinberger, K.~Q. (2017).
\newblock On fairness and calibration.
\newblock In {\em Advances in Neural Information Processing Systems}, pages
  5680--5689.

\bibitem[Pulley et~al., 2012]{pulley2012operational}
Pulley, J.~M., Denny, J.~C., Peterson, J.~F., Bernard, G.~R., Vnencak-Jones,
  C.~L., Ramirez, A.~H., Delaney, J.~T., Bowton, E., Brothers, K., Johnson, K.,
  et~al. (2012).
\newblock Operational implementation of prospective genotyping for personalized
  medicine: the design of the vanderbilt predict project.
\newblock {\em Clinical Pharmacology \& Therapeutics}, 92(1):87--95.

\bibitem[Savage, 1971]{savage1971elicitation}
Savage, L.~J. (1971).
\newblock Elicitation of personal probabilities and expectations.
\newblock {\em Journal of the American Statistical Association},
  66(336):783--801.

\bibitem[Shafer and Vovk, 2008]{shafer2008tutorial}
Shafer, G. and Vovk, V. (2008).
\newblock A tutorial on conformal prediction.
\newblock {\em Journal of Machine Learning Research}, 9(Mar):371--421.

\bibitem[Vovk, 2012]{vovk2012conditional}
Vovk, V. (2012).
\newblock Conditional validity of inductive conformal predictors.
\newblock In {\em Asian conference on machine learning}, pages 475--490.

\bibitem[Vovk et~al., 2005]{vovk2005algorithmic}
Vovk, V., Gammerman, A., and Shafer, G. (2005).
\newblock {\em Algorithmic learning in a random world}.
\newblock Springer Science \& Business Media.

\bibitem[Zadrozny and Elkan, 2001]{zadrozny2001obtaining}
Zadrozny, B. and Elkan, C. (2001).
\newblock Obtaining calibrated probability estimates from decision trees and
  naive bayesian classifiers.
\newblock In {\em Icml}, volume~1, pages 609--616. Citeseer.

\bibitem[Zhao et~al., 2020]{zhao2020individual}
Zhao, S., Ma, T., and Ermon, S. (2020).
\newblock Individual calibration with randomized forecasting.
\newblock {\em arXiv preprint arXiv:2006.10288}.

\bibitem[Zinkevich, 2003]{zinkevich2003online}
Zinkevich, M. (2003).
\newblock Online convex programming and generalized infinitesimal gradient
  ascent.
\newblock In {\em Proceedings of the 20th international conference on machine
  learning (icml-03)}, pages 928--936.

\end{thebibliography}

\onecolumn 
\newpage
\appendix

\section{Additional Results}

\subsection{Multiclass Prediction} 
\label{appendix:multiclass}

For multiclass prediction, we suppose that $Y$ can take $K$ distinct values. We denote $\Delta^K$ as the $K$-dimensional probability simplex. For notational convenience we represent $Y$ as a one-hot vector in $\mathbb{R}^K$, so $\Yc = \lbrace (1, 0, \cdots), (0, 1, \cdots), \cdots \rbrace$. 

\paragraph{Protocol 3: Decision Making with Bets, Multiclass} At time $t = 1, \cdots, T$
\begin{enumerate}
    \item Nature reveals $x_t \in \Xc$ and chooses $\mu_t^* \in \Delta^K$ without revealing it
    \item Forecaster reveals $\mu_t \in \Delta^K$ and  $c_t \in R_+^K$ 
    \item Agent $t$ has loss $l_t: \Yc \times \Ac \to \mathbb{R}$ and chooses action $a_t$ and $g_t \in \mathbb{R}^K$
    \item Sample $y_t \sim \mathrm{Categorical}(\mu_t^*)$ and reveal $y_t$
    \item Agent total loss is $l_t(y_t, a_t) - \langle g_t, y_t - \mu_t \rangle + \langle |g_t|, c_t \rangle$, forecaster loss is $\langle g_t, y_t - \mu_t \rangle - \langle |g_t|, c_t \rangle$
\end{enumerate}

As before we require the regularity condition that $\mu_c + c_t \in [0, 1]^K$ and $\mu_t - c_t \in [0, 1]^K$ (even though these are no longer on $\Delta^K$, hence not probabilities. 

Similar to Section 3 we can denote the agent's maximum /  minimum expected loss under the forecasted probability as
\begin{align*}
    L_t^{\mathrm{max}} &= \max_{\tilde{\mu} \in \Delta^K, \tilde{\mu} \in \mu_t \pm c_t} \Eb_{Y \sim \tilde{\mu}} [l_t(a_t, Y)]  \\ 
    L_t^{\mathrm{min}} &= \min_{\tilde{\mu} \in \Delta^K, \tilde{\mu} \in \mu_t \pm c_t} \Eb_{Y \sim \tilde{\mu}} [l_t(a_t, Y)] 
\end{align*}
and true expected loss as 
$ L_t^* = \Eb_{Y \sim \mu^*_t}[l_t(a_t, Y)] $. As before denote
\begin{align*}
    \mathrm{L}^{\mathrm{pay}}_t = L_t^* + \Eb_{\mu^*}[ \langle g_t, \mu_t - Y \rangle + \langle |g_t|, c_t \rangle ] 
\end{align*}

\begin{prop}
\label{prop:decision_guarantee_multi}
If $g_t = l(\cdot, a_t) - \inf_{\gamma \in \mathbb{R}} \langle c_t, |l - \gamma 1| \rangle $ then $L_t^{\mathrm{pay}} = L_t^{\mathrm{max}}$
\end{prop}
\begin{proof}[Proof of Proposition~\ref{prop:decision_guarantee_multi}]
As a notation shorthand we denote $l_t(a_t, Y)$ with the vector $l$, such that $l_i = l_t(a_t, Y = i)$. We first show a closed form solution for $L_t^{\mathrm{max}}$ which can be written as 
\begin{align*}
    L^{\mathrm{max}}_t &= \sup_{\tilde{\mu} \in \Delta^K, \tilde{\mu} \in \mu_t \pm c_t} \Eb_{Y \sim \tilde{\mu}} [l_t(a_t, Y)] \\
    &= \sup_{\tilde{\mu} \in \Delta^K, \tilde{\mu} \in \mu_t \pm c_t} \langle \tilde{\mu}, l \rangle & \mathrm{Notation\ Change} \\
    &= \langle \mu_t, l\rangle + \sup_{\delta \mu \in [-c_t, c_t], \langle \delta \mu, 1 \rangle = 0} \langle \delta \mu, l \rangle & \mathrm{Algebric\ Manipulation} \\
    &= \langle \mu_t, l\rangle + \sup_{\delta \mu \in [-c_t, c_t]} \inf_{\gamma \in \mathbb{R}} \langle \delta \mu, l \rangle - \gamma \langle \delta \mu, 1 \rangle & \mathrm{Lagrangian} \\
    &= \langle \mu_t, l\rangle +  \inf_{\gamma \in \mathbb{R}} \sup_{\delta \mu \in [-c_t, c_t]} \langle \delta \mu, l \rangle - \gamma \langle \delta \mu, 1 \rangle  & \mathrm{Sion\ Minimax\ Theorem} \\
    &= \langle \mu_t, l\rangle +  \inf_{\gamma \in \mathbb{R}} \langle c_t, |l - \gamma 1| \rangle 
\end{align*}
Similarly we have
\begin{align*}
    L^{\mathrm{min}}_t &= \langle \mu_t, l \rangle  -  \inf_{\gamma \in \mathbb{R}} \langle c_t, |l - \gamma 1| \rangle 
\end{align*}
Denote the $\gamma$ that achieves the infimum as $\gamma^*$. Comparing with $L^{\mathrm{pay}}_t$ we have
\begin{align*}
    L^{\mathrm{pay}}_t &= L_t^* + \Eb_{\mu^*}[ \langle g_t, \mu_t - Y \rangle + \langle |g_t|, c_t \rangle ] \\
    &= \langle \mu^*, l \rangle - \langle l - \gamma^* 1, \mu_t - \mu_t^* \rangle + \langle c_t, |l - \gamma^* 1| \rangle  \\
    &= \langle l, \mu_t \rangle + \langle c_t, |l - \gamma^* 1| \rangle & \langle \mu_t, 1 \rangle = 0, \langle \mu_t^*, 1 \rangle = 0 \\
    &= L^{\mathrm{max}}
\end{align*}

%
\end{proof}

\subsection{Offline Calibration}
\label{appendix:offline}

For this section we restrict to the i.i.d. setup, 
where we assume there are random variables $X, Y$ with some distribution $p^*_{XY}$ such that at each time step, 
\begin{align*}
    x_t \sim X \qquad \mu_t^* = \Eb[Y \mid x_t]
\end{align*}
We also assume that the forecaster 's choice $\mu_t, c_t$ and the agent's choice $b_t$ in Protocol 2 are computed by functions of $x_t$
\begin{align*}
    \mu: x_t \mapsto \mu_t \quad c: x_t \mapsto c_t \quad b: x_t \mapsto b_t
\end{align*}
In other words, given the input $x_t$ all the players choose their actions based on fixed functions of $x_t$. 

The following definition is the equivalent of asymptotic soundness in the i.i.d. setup
\begin{definition}
\label{def:offline_sound}
We say that the functions $\mu, c: \Xc \to [0, 1]$ are sound with respect to some set of functions $\Bc \subset \lbrace \Xc  \to [-M, M] \rbrace$ if 
\[ \sup_{b \in \Bc} \Eb[b(X)(\mu(X) - \Eb[Y \mid X]) - |b(X)|c(X)] \leq 0 \]
If $c(x) \equiv 0$ we say $\mu$ is $\Bc$-calibrated. 
\end{definition}
Intuitively if $\mu, c$ are sound with respect to $\Bc$ then if the decision making agents chooses a strategy in $b \in \Bc$ we can guarantee that the forecaster will not lose (on average). In other words, if $\mu, c$ are sound according to Definition~\ref{def:offline_sound}, and if $b_t = b(x_t)$ for some $b \in \Bc$, then the forecaster is almost surely asymptotically sound as defined in Eq.(\ref{eq:soundness}).




%


\subsubsection{Examples and Special Cases}
\label{appendix:calibration_special_cases}
\paragraph{Standard Calibration} Standard calibration is defined as: for any $u \in [0, 1]$, among the $X$ where $\mu(X) = u$ it is indeed true that $Y$ is $1$ with $u$ probability. Formally this can be written as 
\begin{align*}
    \Eb[Y \mid \mu(X) = u] = u, \forall u \in [0, 1]
\end{align*}
Deviation from this ideal situation is measured by the maximum calibration error (MCE). 
\begin{align*}
    \mathrm{MCE}(\mu) = \max_{u \in [0, 1]} \left\lvert \Eb[ Y \mid \mu(X) = u ] - u \right\rvert 
\end{align*}
Note that the MCE may be ill-defined if there is an interval $(u_0, u_1) \subset [0, 1]$ such that $\mu(X) \in (u_0, u_1)$ with zero probability. We are going to avoid the technical subtlety  by assuming that this does not happen, i.e. the distribution of $\mu(X)$ is supported on the entire set $[0, 1]$.  

When $\Bc$ is the set of all possible functions $\mu(x) \to \mathbb{R}$ (i.e. it only depends on the probability forecast $\mu(x)$ but not $x$ itself), we obtain the standard definition of calibration~\cite{dawid1985calibration,guo2017calibration}, as shown by the following proposition
\begin{restatable}{prop}{standardcalibration}
\label{prop:standard_calibration}
The forecaster function $\mu: \Xc \to [0, 1], c: x \mapsto c_0$ is sound with respect to $\Bc = \lbrace x \mapsto \tilde{b}(\mu(x)), \tilde{b}: \mathbb{R} \to \mathbb{R} \rbrace$ if and only if the MCE error of $\mu$ is less than $c_0$. 
\end{restatable}
\begin{proof} See Appendix~\ref{appendix:proof} \end{proof}

We remark that this proposition (intentionally) does not involve the upper bound $M$ on $b$; it holds even when $M \to \infty$. 

\paragraph{Multi-Calibration}
Multi-calibration~\cite{hebert2017calibration} achieves standard calibration for all subsets $S$ in some collection of sets $\Sc$. The following proposition shows that a forecaster that's sound with respect to any function that only depends on $\mu(x)$ and takes zero value whenever $x \not\in S$ is also multicalibrated.  

\begin{restatable}{prop}{multicalibration}
Let $\Sc \subset 2^{\Xc}$. If a forecaster function $\mu: \Xc \to [0, 1], c: x \mapsto c_0$ is sound with respect to $\Bc = \lbrace x \mapsto \tilde{b}(\mu(x))\mathbb{I}(x \in S),  S \in \Sc, \tilde{b}: \mathbb{R} \to \mathbb{R} \rbrace$, then it is $(\Sc, c_0)$-multicalibrated. 
\end{restatable}

\section{Experiment Details and Additional Results}

\subsection{Airline Delay}
\label{appendix:airline}

\paragraph{Negative $c_t$} In Protocol 2 $c_t$ must be non-negative for its interpretation as a probability interval $[\mu_t - c_t, \mu_t + c_t]$. However if we only consider the flight delay insurance interpretation: airline pay passenger $b_t^1$ if flight delays and passenger pays airline $
    b_t^0 :=  \frac{b_t^1(\mu_t + c_t)}{1 - \mu_t - c_t}
$ if flight doesn't delay. These payments are meaningful for both positive and negative $c_t$; the passenger utility (with insurance) can be computed as $r^{\mathrm{trip}} - c^{\mathrm{ticket}} - (\mu_t + c_t) c^{\mathrm{delay}}$, which is also meaningful for both positive and negative $c_t$. We find that allowing negative $c_t$ improves the stability of the algorithm.

\paragraph{Passenger Model}

We sample $r^{\mathrm{alt}}$ as $\mathrm{Uniform}(0, 200)$ and sample $r^{\mathrm{trip}}$ from $\mathrm{Uniform}(0, 400)$. We assume the cost of delay can be more varied, so we sample it from the following process: $z \sim \mathrm{Uniform}(4, 9)$ and $c^{\mathrm{delay}} = 0.2e^{z} $. This gives us a cost of delay between $[10, 1600]$, but large values are less likely. 

\paragraph{Additional Results}

We show additional comparison with other alternatives to Algorithm~\ref{alg:swap_minimization} in Figure~\ref{fig:airline_extension}. For details about these alternatives see Section~\ref{sec:additional_experiment}.

\begin{figure*}
    \centering
    \includegraphics[width=0.7\linewidth]{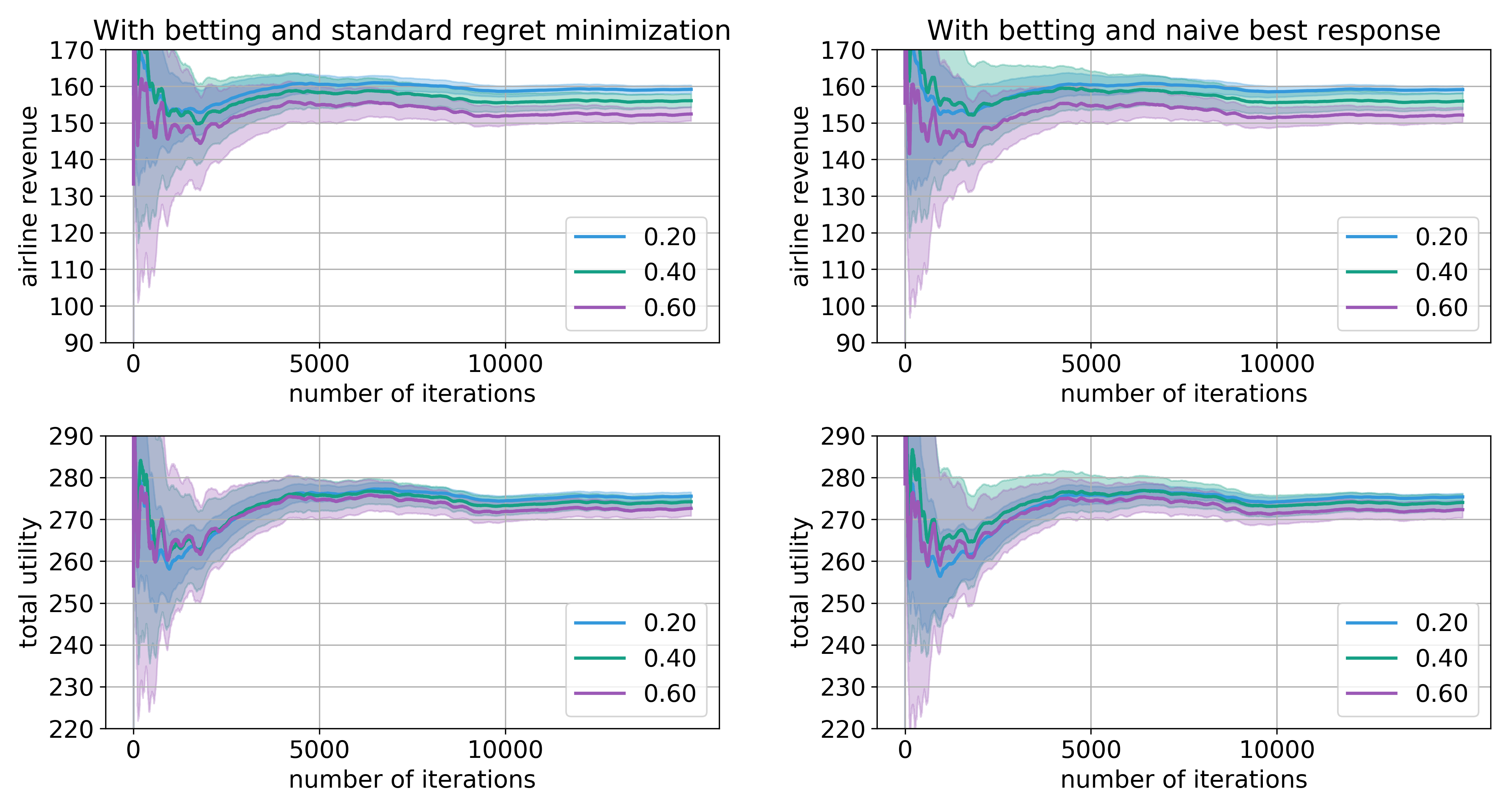}
    \caption{This plot extends Figure~\ref{fig:airline}. We compare with additional Alternatives to Algorithm~\ref{alg:swap_minimization}. 
    }
    \label{fig:airline_extension}
\end{figure*}

\subsection{Additional Experiments}
\label{appendix:additional_experiments}

\paragraph{Decision Loss} For each data point we associate an extra feature $z$ used to define decision loss. For MNIST this is the digit label and for UCI Adult this is the age (binned by quantile into 10 bins). We simulate three kinds of decision losses; for each type of decision loss we randomly sample a few instantiations. 

1. One-sided: we assume that $a \in [0, 1]$ and each decision loss $l(z, y, a)$ is large if $y \neq a$ and small if $y = a$. For different values of $z$ there are different stakes (i.e. how much does the loss when $y = a$ differ from $y \neq a$). 

2. Different Stakes: Each value of the decision loss $l(z, y, a)$ is a draw from $\mathcal{N}(0, z)$, which is used to capture the feature that certain groups of people have larger stakes

3. Random. Each value of the decision loss $l(z, y, a)$ is a draw from $\mathcal{N}(0, 10)$ but clipped to be within $[-10, 10]$.

\paragraph{Forecasted Loss vs. True Loss} In Figure~\ref{fig:true_utility} we plot the relationship between the expected loss under the forecasted probability and the expected loss under the true probability (we can compute this for the MNIST dataset because the true probability is known as explained in Section~\ref{sec:additional_experiment_setup}). Even if we apply histogram binning recalibration (explained in Section~\ref{sec:additional_experiment_setup}), the individual probabilities are almost always incorrect. 

\paragraph{Asymptotic Exactness}  In Figure~\ref{fig:betting_payoff} and Figure~\ref{fig:betting_payoff2} we plot the average betting loss of the forecaster. Algorithm~\ref{alg:online_prediction} consistently achieve better asymptotic exactness compared to alternatives.

\paragraph{Average Interval Size} In Figure~\ref{fig:betting_cost} we plot the interval size $c_t$. A small $c_t$ satisfies desideratum 2 in Section 3 and makes the guarantee in Proposition~\ref{prop:decision_guarantee} useful for decision makers. We observe that most interval sizes are small, and larger intervals are exponentially unlikely. 

\begin{figure*}
    \centering
    \begin{tabular}{c}
    \includegraphics[width=0.85\linewidth]{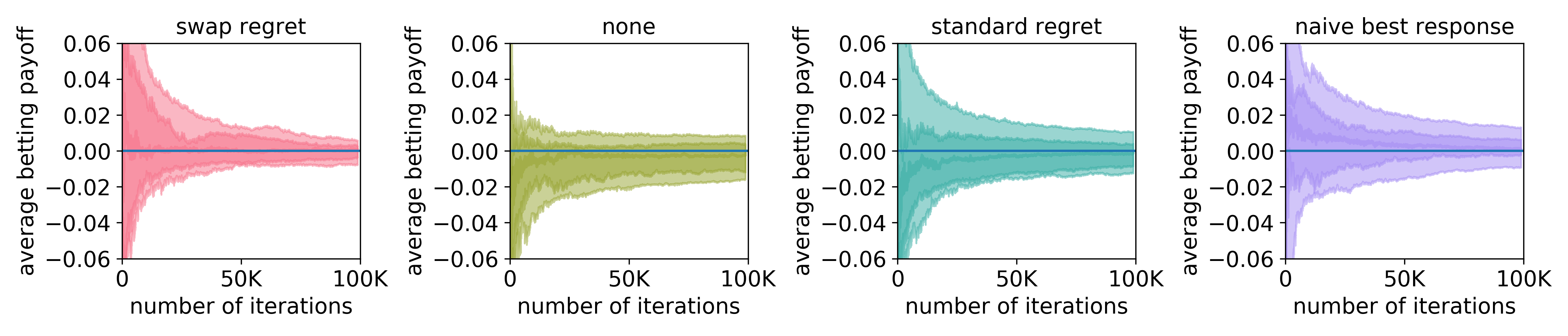} \\
    \end{tabular}
    \caption{This plot is identical to Figure~\ref{fig:betting_payoff} but for the Adult dataset}  
    \label{fig:betting_payoff2}
\end{figure*}

\begin{figure}
    \centering
    \includegraphics[width=0.5\linewidth]{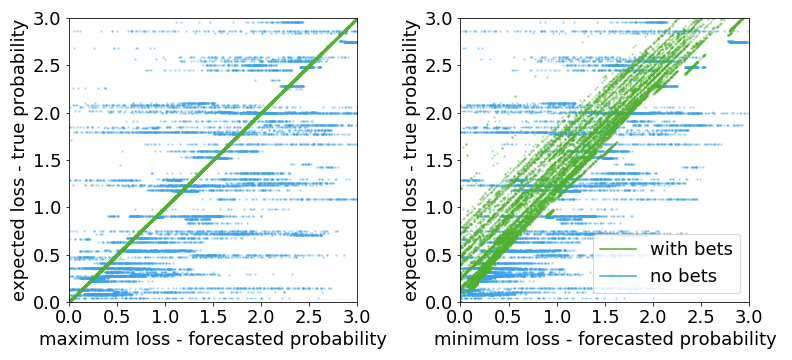}
    \caption{The expected loss under the forecaster utility vs. expected loss under the true probability. Each dot represents an individual probability forecast with a particular choice of loss function. We use histogram binning on the entire validation set to recalibrate the forecaster. Even though the forecaster is calibrated, the individual probabilities are often incorrect. Therefore, the expected loss under the forecasted probability often differs from the expected loss under the true probability (blue dots). On other hand, with additional payment from the bets, the expected total loss under true probability is always bounded between the minimum loss under the forecasted probability, and the maximum loss under the forecasted probability.}
    \label{fig:true_utility}
\end{figure}

\allowdisplaybreaks

\section{Proof of Theorem~\ref{thm:online_sound_forecast}} 
\label{appendix:swap}

Algorithm~\ref{alg:swap_minimization} is the core reason why our forecasting algorithm achieves Theorem~\ref{thm:online_sound_forecast}, so before we prove Theorem~\ref{thm:online_sound_forecast} we first understand Algorithm~\ref{alg:swap_minimization}. The goal of Algorithm~\ref{alg:swap_minimization} is to select a sequence of $\lambda_t$ to minimize the loss
$
    \sum_{t=1}^T (r_t + s_t \lambda_t)^2 
$ for any choice of $r_t, s_t \in \mathbb{R}$. More specifically the goal is to minimize the swap regret defined by
\begin{align*}
    \rswapt = \underbrace{\sum_{t=1}^T (r_t + s_t \lambda_t)^2}_{\textrm{Loss incurred by Algorithm~\ref{alg:swap_minimization}}} - \inf_{\psi \in L^1[-1, 1]} \underbrace{\sum_{t=1}^T (r_t + s_t \psi(\lambda_t))^2}_{\textrm{Loss incurred by ``alternative'' } \psi(\lambda_t)} \numberthis\label{eq:swap_regret}
\end{align*}
where $L^1[-1, 1]$ denotes the set of $1$-Lipshitz functions $\mathbb{R} \to [-1, 1]$. 
Intuitively, $ \sum_{t=1}^T (r_t + s_t \psi(\lambda_t))^2$ is the loss of an alternative algorithm: whenever Algorithm~\ref{alg:swap_minimization} selects $\lambda_t$,  select $\psi(\lambda_t)$ instead. Swap regret measures the additional loss compared to the best alternative algorithm. 
We remark that if instead Algorithm~\ref{alg:swap_minimization} minimizes the standard regret, we can no longer guarantee  Theorem~\ref{thm:online_sound_forecast}. For intuition on the reason we refer interested readers to a counter-example in \cite{cesa2006prediction} Section 4.5 (for a related calibration problem). 

We now prove that Algorithm~\ref{alg:swap_minimization} indeed achieve its goal of minimizing the swap regret. 
\begin{restatable}{theorem}{swapminimization}
\label{thm:swap_minimization}
If there exists $M_1, M_2$ such that $\forall t, |s_t| \leq M_1, |r_t/s_t| \leq M_2$, then there exists a constant $C(M_1, M_2) > 0$, such that for any choice of $K > 1$, the regret of Algorithm~\ref{alg:swap_minimization} is bounded by 
\begin{align*}
    \rswapt \leq C(M_1, M_2) K^2\log T + \frac{1}{K^2} \sum_{t=1}^T s_t^2 
\end{align*}
\end{restatable}

In particular, if we choose $K^2 = \sqrt{T / \log T}$ then the swap regret $\rswapt$ is bounded by $O(\sqrt{T \log T})$.

Before we prove Theorem~\ref{thm:swap_minimization} we show how to use it to prove Theorem~\ref{thm:online_sound_forecast} restated below. 
\onlinesoundness*

\begin{proof}[Proof of Theorem~\ref{thm:online_sound_forecast}]
To prove this theorem we need the following inequality that relates the LHS in Eq.(\ref{eq:online_sound_forecast}) to the swap regret $\rswapt$
\begin{restatable}{lemma}{swapboundsound}
\label{lemma:swap_bound_sound}
For any choice of $r_t, s_t, \lambda_t, t=1, \cdots, T$ we have
\begin{align*}
     \left( \frac{1}{T} \sum_{t=1}^T s_t(r_t + s_t \lambda_t) \right)^2 \leq \frac{\rswapt}{T^2} \sum_{t=1}^T s_t^2 
\end{align*}
\end{restatable}
Because at each iteration Algorithm~\ref{alg:online_prediction} selects $r_t = \frac{b_t}{\sqrt{|b_t|}} (\mu_t - y_t) - \sqrt{|b_t|} \hc_t $ and $s_t = - \sqrt{|b_t|}$ we can plug this into Lemma~\ref{lemma:swap_bound_sound} and conclude that for any sequence of $\lambda_t$ (which includes any $\lambda_t$ chosen by Algorithm~\ref{alg:swap_minimization}), Algorithm~\ref{alg:online_prediction} must satisfy
\begin{align*}
    \left(\frac{-1}{T} \sum_{t=1}^T b_t (\mu_t - y_t) + |b_t|\hc_t + |b_t| \lambda_t\right)^2 \leq \frac{\rswapt}{T} \frac{1}{T} \sum_{t=1}^T |b_t|  \leq \frac{M\rswapt}{T}
\end{align*}
In addition we have 
\begin{align*}
    \left\lvert \frac{r_t}{s_t} \right\rvert = \left\lvert - \frac{b_t}{|b_t|}(\mu_t - y_t) + \hc_t \right\rvert \leq 2
\end{align*}
So the conditions of Theorem~\ref{thm:swap_minimization} is satisfied (i.e. $|s_t|$ and $|r_t/s_t|$ are bounded), and we can apply Theorem~\ref{thm:swap_minimization} to conclude $\rswapt = O(\sqrt{T  \log T})$. Combined we have
\begin{align*}
 \left(\frac{1}{T} \sum_{t=1}^T b_t (\mu_t - y_t) - |b_t| (\hc_t + \lambda_t) \right)^2 = O(M\sqrt{T \log T}/T) = O(\sqrt{\log T / T})
\end{align*}
\end{proof}

Now we proceed to prove Theorem~\ref{thm:swap_minimization}
\begin{proof}[Proof of Theorem~\ref{thm:swap_minimization}]
To prove this theorem we first need the following Lemma, which bounds the standard regret (rather than swap regret) 
\begin{restatable}{lemma}{standardregretbound}
\label{lemma:standard_regret_bound}
If there exists some $M_1, M_2 > 0$ such that $\forall t$, $| \beta_t| \leq M_1$ and $|\alpha_t/\beta_t| \leq M_2$, choosing $\lambda_t = \arg\inf_{\lambda \in \mathbb{R}} \sum_{\tau=1}^{t-1} (\alpha_\tau + \beta_\tau \lambda)^2$ satisfies for some constant $C(M_1, M_2) > 0$
\begin{align*}
    \sum_{t=1}^T (\alpha_t + \beta_t \lambda_t)^2 \leq \inf_\lambda \sum_{t=1}^T (\alpha_t + \beta_t \lambda)^2 + C(M_1, M_2) \log T 
\end{align*}
\end{restatable}


To prove Theorem~\ref{thm:swap_minimization} we first bound the discretized swap regret, defined as follows
\begin{align*}
    \rswapdt = \sum_{t=1}^T (r_t + s_t \lambda_t)^2 - \sum_{k=1}^K \inf_\lambda \sum_{t=1}^T \mathbb{I}(\lambda_t \in [v_k, v_{k+1}))(r_t + s_t \lambda)^2 
\end{align*}
Intuitively, this is the regret with respect to the alternative algorithm: whenever the Algorithm~\ref{alg:swap_minimization} chooses some $\lambda_t$ that falls with in a bin $[v_k, v_{k+1})$, choose a different $\lambda$. 

To bound the discretized swap regret our proof strategy is similar to \cite{blum2007external}: As a notation shorthand we denote $c^t(\lambda) = (r_t + s_t \lambda)^2$ and $\Ic_k = [v_k, v_{k+1})$. For each $k$, we apply Lemma~\ref{lemma:standard_regret_bound} with $\alpha_t = r_t \mathbb{I}(\lambda_t \in \Ic_k)$ and $\beta_t = s_t \mathbb{I}(\lambda_t \in \Ic_k)$ with the convention that $0/0 = 0$. By the assumptions in Theorem~\ref{thm:swap_minimization} we know that $|\beta_t| \leq M_1$ and $|\alpha_t/\beta_t| \leq M_2$ so we can guarantee by Lemma~\ref{lemma:standard_regret_bound} that
\begin{align*}
    \sum_{t=1}^T \mathbb{I}(\lambda_t \in \Ic_k) c^t(\lambda_t^k) &= \sum_{t=1}^T (\alpha_t + \beta_t \lambda^k_t)^2 \\
    &\leq \inf_\lambda \sum_{t=1}^T (\alpha_t + \beta_t \lambda)^2 + C(M_1, M_2) \log T =  \inf_\lambda \sum_{t=1}^T \mathbb{I}(\lambda_t \in \Ic_k) c^t(\lambda) + C(M_1, M_2) \log T
\end{align*}
The total loss is given by 
\begin{align*}
    \sum_{t=1}^T c^t(\lambda_t) = \sum_{t=1}^T \sum_{k=1}^K \mathbb{I}(\lambda_t \in \Ic_k) c^t(\lambda^k_t) \leq \sum_{k=1}^K \inf_\lambda \sum_{t=1}^T \mathbb{I}(\lambda_t \in \Ic_k) c^t(\lambda) + C(M_1, M_2)K\log T
\end{align*}
We can conclude that 
\begin{align*}
    \rswapdt := \sum_{t=1}^T c^t(\lambda_t) - \sum_{k=1}^K \inf_\lambda \sum_{t=1}^T \mathbb{I}(\lambda_t \in \Ic_k) c^t(\lambda) \leq C(M_1, M_2)K\log T
\end{align*}

Finally we conclude the proof of the theorem with the following Lemma that bounds the difference between the discretized swap regret and the continuous swap regret. 
\begin{lemma}\label{lemma:discrete_to_continuous}
In Algorithm~\ref{alg:swap_minimization}, $\rswapt  \leq \rswapdt + \sum_{t=1}^T s_t^2 \frac{v_K - v_0}{K} $
\end{lemma}
\end{proof}

\begin{proof}[Proof of Lemma~\ref{lemma:discrete_to_continuous}]
Denote $\Ic_k = [v_k, v_{k+1})$ and denote $\delta v = \max_k v_{k+1} - v_k$. In addition denote $\lambda^*_k = \arg\inf_\lambda \sum_{t=1}^T \mathbb{I}(\lambda_t \in \Ic_k)(r_t + s_t \lambda)^2$
\begin{align*}
    &\rswapt - \rswapdt \\
    &= \sum_{k=1}^K \inf_\lambda \sum_{t=1}^T \mathbb{I}(\lambda_t \in \Ic_k)(r_t + s_t \lambda)^2  - \inf_{\psi \in L^1} \sum_{t=1}^T (r_t + s_t \psi(\lambda_t))^2 &\text{Definition}  \\
    &=  \sum_{k=1}^K  \inf_\lambda \sum_{t=1}^T \mathbb{I}(\lambda_t \in \Ic_k)(r_t + s_t \lambda)^2 - \inf_{\psi \in L^1} \sum_{k=1}^K \sum_{t=1}^T \mathbb{I}(\lambda_t \in \Ic_K)(r_t + s_t \psi(\lambda_t))^2    & \text{Decompose\ 2nd\ term} \\
     &\leq \sum_{k=1}^K \left(\inf_\lambda \sum_{t=1}^T \mathbb{I}(\lambda_t \in \Ic_k)(r_t + s_t \lambda)^2 - \inf_{\psi \in L^1} \sum_{t=1}^T \mathbb{I}(\lambda_t \in \Ic_K)(r_t + s_t \psi(\lambda_t))^2   \right) & \text{Jensen} \\
    &= \sum_{k=1}^K \left( \sum_{t=1}^T \mathbb{I}(\lambda_t \in \Ic_k)(r_t + s_t \lambda^*_k)^2 - \inf_{\delta \psi \in L^1} \sum_{t=1}^T \mathbb{I}(\lambda_t \in \Ic_K)(r_t + s_t (\lambda^*_k + \delta \psi(\lambda_t)))^2   \right) & \text{Change\ of\ variable}\\
    &\leq \sum_{k=1}^K \sum_{t=1}^T \mathbb{I}(\lambda_t \in \Ic_k) s_t^2 \delta_v^2 & \text{1-Lipschitzness}\\
    &= \sum_{t=1}^T s_t^2 \delta_v^2 
\end{align*}
For Algorithm~\ref{alg:swap_minimization} we know that $\delta v = \frac{v_K - v_0}{K}$ because of the equal width partition. 
\end{proof}




\standardregretbound*
\begin{proof}[Proof of Lemma~\ref{lemma:standard_regret_bound}]
The proof strategy is similar to Chapter 4 of \cite{cesa2006prediction}. 
Define $\lambda_{t}^* = \arg\inf_\lambda \sum_{\tau=1}^{t} (\alpha_\tau + \beta_\tau \lambda)^2$. In words the only difference between $\lambda_{t}^*$ and $\lambda_{t}$ is that $\lambda^*_{t}$ can look one step into the future. Then by Lemma 3.1 of \cite{cesa2006prediction} we have 
\begin{align*}
    R_T := \sum_{t=1}^T (\alpha_t + \beta_t \lambda_t)^2 - \inf_\lambda \sum_{t=1}^T (\alpha_t + \beta_t \lambda)^2 \leq \sum_{t=1}^T (\alpha_t + \beta_t \lambda_t)^2 -(\alpha_t + \beta_t \lambda_t^*)^2 \numberthis\label{eq:lemma_standard_2}
\end{align*}


We introduce simplified notation $r_t(\lambda) = \sum_{\tau=1}^t (\alpha_\tau + \beta_\tau \lambda)^2$. So with the new notation $\lambda_t = \inf_\lambda r_{t-1}(\lambda)$ and $\lambda_t^* = \inf_\lambda r_t(\lambda)$. We can compute 
\begin{align*}
    r_{t-1}'(\lambda_t) = 0, \qquad r_{t-1}''(\lambda_t) = 2\sum_{\tau=1}^{t-1} \beta_\tau^2, \qquad r_{t-1}'''(\lambda) = 0  \numberthis\label{eq:lemma_standard_1}
\end{align*}
Also denote $\delta \lambda_t = \lambda_t^* - \lambda_t$ we have
\begin{align*}
    \delta \lambda_t &= \arg\inf_{\delta \lambda}  r_{t}(\lambda_t + \delta \lambda) \\
    &= \arg\inf_{\delta \lambda} r_{t-1}(\lambda_t + \delta \lambda) + (\alpha_t + \beta_t \lambda_t + \beta_t \delta\lambda)^2 & \text{By\ definition} \\
    &= \arg\inf_{\delta \lambda} r_{t-1}(\lambda_t) + r_{t-1}'(\lambda_t) \delta\lambda + \frac{1}{2} r_{t-1}''(\lambda_t) \delta\lambda^2 + \\
    &\qquad \qquad (\alpha_t + \beta_t\lambda_t)^2 + 2(\alpha_t + \beta_t\lambda_t)\beta_t \delta \lambda + \beta_t^2 \delta\lambda^2 & \text{Taylor\ expansion} \\
    &= \arg\inf_{\delta \lambda} \sum_{\tau=1}^{t-1} \beta_\tau^2\delta\lambda^2 + 2(\alpha_t + \beta_t\lambda_t)\beta_t \delta \lambda + \beta_t^2 \delta\lambda^2 & \text{Apply\ Eq.(\ref{eq:lemma_standard_1})\ and\ remove\ irrelevant\ terms} \\
    &= - \frac{2(\alpha_t + \beta_t\lambda_t) \beta_t} {2\sum_{\tau=1}^{t-1} \beta_\tau^2 + 2\beta_t^2} = - \frac{(\alpha_t + \beta_t \lambda_t) \beta_t} {\sum_{\tau=1}^{t} \beta_\tau^2} & \text{Minimizer\ of\ quadratic}
\end{align*}

Applying the new result to Eq.(\ref{eq:lemma_standard_2}) we have
\begin{align*}
    R_T &\leq \sum_{t=1}^T  (\alpha_t + \beta_t  \lambda_t)^2 - (\alpha_t + \beta_t \lambda_t^*)^2 \\
    &= \sum_{t=1}^T  (2\alpha_t + \beta_t \lambda_t + \beta_t\lambda^*_t)(\beta_t \lambda_t - \beta_t \lambda^*_t) & \text{By\ } (a+b)(a-b) = a^2-b^2 \\
    &= \sum_{t=1}^T \left( |\alpha_t + \beta_t \lambda_t| + |\alpha_t + \beta_t \lambda_t^*| \right) |\beta_t \delta \lambda_t | & \text{Cauchy\ schwarz} \\
    &= \sum_{t=1}^T 2|\alpha_t + \beta_t \lambda_t| |\beta_t \delta \lambda_t | & \mathrm{By\ } (\alpha_t + \beta_t  \lambda^*_t)^2 \leq (\alpha_t + \beta_t  \lambda_t)^2  \\
    &= \sum_{t=1}^T 2|\alpha_t + \beta_t \lambda_t| \left\lvert \beta_t \frac{(\alpha_t + \beta_t \lambda_t) \beta_t} {\sum_{\tau=1}^{t} \beta_\tau^2}  \right\rvert & \text{Insert\ expression\ for\ } \delta \lambda_t \\ 
    &\leq \sum_{t=1}^T 2(\alpha_t + \beta_t \lambda_t)^2  \frac{\beta_t^2} {\sum_{\tau=1}^{t} \beta_\tau^2} & \text{Cauchy\ schwarz}  \\ 
    &= \sum_{t=1}^T 2(\alpha_t/\beta_t + \lambda_t)^2  \frac{\beta_t^4} {\sum_{\tau=1}^{t} \beta_\tau^2}  \\
    &\leq \left( \max_t 2(\alpha_t/\beta_t + \lambda_t)^2 \right) \sum_{t=1}^T \frac{\beta_t^4} {\sum_{\tau=1}^{t} \beta_\tau^2} & \text{Holder\ inequality}  \\
    &\leq 8M_2^2 \sum_{t=1}^T \frac{\beta_t^4} {\sum_{\tau=1}^{t} \beta_\tau^2} & |\lambda_t| \leq M_2
\end{align*}
Finally we apply the Lemma~\ref{lemma:sum_algebra} to conclude that 
\begin{align*}
    R_T \leq 8 M_2^2 M_1^2 \log (T+1) 
\end{align*}

\begin{restatable}{lemma}{sumalgebra}
\label{lemma:sum_algebra}
For any sequence $\beta_t, t=1, \cdots, T$ such that $|\beta_t| \leq M, \forall t$ we have $\sum_{t=1}^T \frac{\beta_t^4}{\sum_{\tau=1}^t \beta_\tau^2} \leq M^2 \log (T+1)$
\end{restatable}
\end{proof}

Finally we prove the remaining unproved Lemmas

\swapboundsound*
\begin{proof}[Proof of Lemma~\ref{lemma:swap_bound_sound}]
Without loss of generality assume $\frac{1}{T} \sum_{t=1}^T s_t(r_t + s_t \lambda_t) > 0$, find some $\epsilon > 0$ such that 
\begin{align*}
    \sum_{t=1}^T s_t(r_t + s_t \lambda_t) = \sum_{t=1}^T s_t^2 \epsilon
\end{align*}
Such an $\epsilon$ can always be found because the range of the RHS is $[0, +\infty)$ as $\epsilon \in [0, +\infty)$ (unless all the $s_t$ are zero, in which case the Lemma is trivially true). Therefore, there must be a solution to the equality. Because the function $\lambda_t \mapsto \lambda_t + \lambda$ is 1-Lipshitz, we have
\begin{align*}
    \rswapt &\geq  \sum_{t=1}^T ( r_t + s_t \lambda_t)^2 - \inf_{\lambda}  \sum_{t=1}^T  ( r_t + s_t (\lambda_t + \lambda) )^2  & \text{Choose\ a\ particular\ } \psi \\
    &\geq  \sum_{t=1}^T ( r_t + s_t \lambda_t )^2 - \sum_{t=1}^T  ( r_t + s_t (\lambda_t - \epsilon))^2 & \text{Choose\ a\ particular\ } \lambda \\
    &= \sum_{t=1}^T (2r_t + 2s_t \lambda_t - s_t\epsilon) s_t \epsilon = 2\left(\sum_{t=1}^T  s_t(r_t + s_t \lambda_t)\right) \epsilon - \sum_t s_t^2 \epsilon^2 = \sum_t s_t^2 \epsilon^2
\end{align*}
Therefore we have
\begin{align*}
    \left( \frac{1}{T} \sum_{t=1}^T s_t(r_t + s_t \lambda_t)  \right)^2 = \frac{1}{T^2} \left( \sum_{t=1}^T s_t^2  \right)^2 \epsilon^2 \leq  \frac{\rswapt}{T^2} \sum_{t=1}^T s_t^2 
\end{align*}
\end{proof}

\sumalgebra*
\begin{proof}[Proof of Lemma~\ref{lemma:sum_algebra}] 
First observe that for any $j$ if we fix the values of $\beta_t, t \neq j$, then choosing $\beta_j = M$ always maximizes $\sum_{t=1}^T \frac{\beta_t^4}{\sum_{\tau=1}^t \beta_\tau^2}$. Therefore, we have
\begin{align*}
    \sum_{t=1}^T \frac{\beta_t^4} {\sum_{\tau=1}^{t} \beta_\tau^2} \leq \sum_{t=1}^T \frac{M^4}{\sum_{\tau=1}^t M^2} = M^2  \sum_{t=1}^T \frac{1}{t}  \leq M^2  \int_{t=1}^{T+1} \frac{1}{t}  = M^2 \log (T+1)
\end{align*}
\end{proof}

\onlinesoundnesstwo*
\begin{proof}[Proof of Corollary~\ref{cor:online_sound_forecast}]
We make a small modification in Algorithm~\ref{alg:online_prediction}. Originally line 5 of Algorithm~\ref{alg:online_prediction} outputs $\mu_t = \hat{\mu}_t$ and $c_t = \hat{c}_t + \lambda_t$; instead we output $\mu'_t = \hat{\mu_t} - (\hat{c}_t + \lambda_t)$ and $c'_t = 0$. 

This modified algorithm can achieve asymptotic exactness because
\begin{align*}
    \frac{1}{T} \sum_{t=1}^T b_t(\mu'_t - y_t) - |b_t|c'_t &= \frac{1}{T} \sum_{t=1}^T b_t(\mu'_t - y_t) & c_t'\mathrm{\ is\ zero} \\
    &= \frac{1}{T} \sum_{t=1}^T b_t(\mu_t - c_t - y_t) & \mathrm{Definition\ of\ } \mu_t' \\
    &= \frac{1}{T} \sum_{t=1}^T b_t(\mu_t - y_t) - b_t c_t \\
    &= \frac{1}{T} \sum_{t=1}^T b_t(\mu_t - y_t) - |b_t| c_t  & b_t \geq 0
\end{align*}
The final expression goes to $0$ by Theorem~\ref{thm:online_sound_forecast}.
\end{proof}
\section{Additional Proofs}
\label{appendix:proof}

\losseqprob*
\begin{proof}[Proof of Proposition~\ref{prop:loss_eq_prob}]
Part I: without loss of generality assume $l_t(a_t, 1) > l_t(a_t, 0)$, denote $L_t = \Eb_{\mu}[l_t(a_t, Y)]$ and we also use the notation shorthand $l_t(y)$ to denote $l_t(a_t, y)$. Since $\mu^* \in [\mu_t - c_t, \mu_t + c_t]$ we have
\begin{align*}
    |L_t - L_t^*| &\leq \sup_{\mu^* \in \mu_t \pm c_t} \left\lvert \Eb_{Y \sim \mu_t}[l_t(Y)] - \Eb_{Y \sim \mu^*}[l_t(Y)] \right\rvert  \\
    &= \sup_{\mu^* \in \mu_t \pm c_t} \left\lvert \mu_t l_t(1) + (1-\mu_t) l_t(0) - \mu^*_t l_t(1) - (1-\mu^*_t) l_t(0) \right\rvert \\
    &= \sup_{\mu^* \in \mu_t \pm c_t} \left\lvert (\mu_t - \mu^*_t) (l_t(1) - l_t(0)) \right\rvert  \\
    &\leq c_t (l_t(1) - l_t(0))
\end{align*}
by similar algebra as above we also have
\begin{align*}
    L_t - L_t^{\mathrm{min}} &= c_t (l_t(1) - l_t(0)) \\
    L_t^{\mathrm{max}} - L_t &= c_t (l_t(1) - l_t(0))
\end{align*}
therefore it must be that $L_t^* \geq L_t^{\mathrm{min}}$ and $L_t^* \leq L_t^{\mathrm{max}}$. 

Part II: Choose $\ell_t(a_t, y) = \alpha y + \beta$ where $\alpha \neq 0$; by choosing $\alpha, \beta$ this can represent any loss function $\ell$ where $\ell(a_t, 0) \neq \ell(a_t, 1)$. We prove the case where $\alpha > 0$ and the case where $\alpha < 0$ can be similarly proven. Suppose $\mu_t^* < \mu_t - c_t$
\begin{align*}
    L_t^* &= \Eb_{Y \sim \mu^*}[\alpha Y + \beta ] = \alpha \mu_t^* + \beta  < \alpha(\mu_t - c_t) + \beta \\
    L_t^{\mathrm{min}} &= \min_{\tilde{\mu} \in \mu_t \pm c_t} \Eb_{Y \sim \tilde{\mu}}[\alpha Y + \beta] = \Eb_{Y \sim \mu_t - c_t}[\alpha Y+ \beta] = \alpha(\mu_t - c_t) + \beta
\end{align*} 
but this would imply that $L_t^* < L_t^{\mathrm{min}}$. 

Suppose $\mu_t^* > \mu_t + c_t$
\begin{align*}
    L_t^* &= \Eb_{\mu^*}[\alpha Y+\beta] = \alpha \mu_t^* + \beta > \alpha(\mu_t + c_t) + \beta  \\
    L_t^{\mathrm{max}} &= \max_{\tilde{\mu} \in \mu_t \pm c_t} \Eb_{Y \sim \tilde{\mu}}[\alpha Y + \beta] = \Eb_{Y \sim \mu_t + c_t}[\alpha Y+ \beta] = \alpha (\mu_t + c_t) + \beta
\end{align*}
but this would imply that $L_t^* > L_t^{\mathrm{max}}$. 
\end{proof}


\probabilitybet*
\begin{proof}[Proof of Lemma~\ref{lemma:probability_vs_bet}]
If: if for some $b \in \mathbb{R}$ we have $f(y) \leq b(y - \mu) - |b|c$ then for any $\tmu$ such that $\tmu \in [\mu-c, \mu+c]$ or equivalently $|\tmu - \mu| \leq c$ we have
\begin{align*}
     \Eb_{Y \sim \tmu}[f(Y)] \leq  \Eb_{Y \sim \tmu} [b(Y - \mu) - |b|c] =  b (\tmu - \mu) - |b| c \leq |b| |\tmu - \mu | - |b|c \leq 0 
\end{align*}
Only if: If $\mu=1$ or $\mu=0$ then the proof is trivial; we consider the case where $\mu \in (0, 1)$. Suppose for any $\tilde{\mu} \in [\mu - c, \mu + c]$ we have $ \Eb_{Y \sim \tmu}[f(Y)] \leq 0$ we have (by instantiating a few concrete values for $\tmu$)
\begin{align*}
    &f(1) (\mu - c) + f(0) (1 - \mu + c) \leq 0 \numberthis\label{eq:lemma1-1} \\
    &f(1) (\mu + c) + f(0) (1 - \mu - c) \leq 0 \numberthis\label{eq:lemma1-2}
\end{align*}
Choose some $b$ such that $f(1) = b(1 - \mu) - |b|c$. Such a $b$ must exist because the range of $b \mapsto b(1 - \mu) - |b| c$ is $\mathbb{R}$. If $b < 0$ then by Eq.(\ref{eq:lemma1-1}) we have
\begin{align*}
    b(1 - \mu + c) (\mu - c) + f(0) (1 - \mu + c) \leq 0, \qquad f(0) \leq - b (\mu - c) = b (0 - \mu) - |b| c
\end{align*}
Conversely if $b \geq 0$ by Eq.(\ref{eq:lemma1-2}) we have
\begin{align*}
    b(1 - \mu - c) (\mu - c) + f(0) (1 - \mu + c) \leq 0, \qquad f(0) \leq - b (\mu + c) = b(0 - \mu) - |b| c
\end{align*}
In either cases this is equivalent to $\forall y \in \lbrace 0, 1 \rbrace$, $f(y) \leq b(y - \mu) - |b| c$.
\end{proof}


\decisionguarantee*
\begin{proof}[Proof of Proposition~\ref{prop:decision_guarantee}]
For convenience denote $l(Y) := l_t(Y, a_t)$. Without loss of generality assume $l(1) > l(0)$
\begin{align*}
    L^{\mathrm{min}}_t &= \min_{\tilde{\mu} \in \mu_t \pm c_t} \Eb_{\tilde{\mu}} [l(Y)] = (\mu_t - c_t) l(1) + (1 - \mu_t + c_t) l(0) = \mu_t l(1) + (1-\mu_t) l(0) - (l(1) - l(0)) c_t  \\
    &= \mu_t^* l(1) + (\mu_t - \mu_t^*) l(1) + (1 - \mu_t^*) l(0) - (\mu_t - \mu_t^*) l(0) - (l(1) - l(0)) c_t  \\
    &= \Eb_{\mu_t^*}[l(Y)] - (l(1) - l(0)) \Eb_{\mu_t^*} [Y - \mu]  - (l(1) - l(0)) c_t] \leq L^{\mathrm{pay}}_t 
\end{align*}
and 
\begin{align*}
    L^{\mathrm{max}}_t &= \min_{\tilde{\mu} \in \mu_t \pm c_t} \Eb_{\tilde{\mu}} [l(Y)] = (\mu_t + c_t) l(1) + (1 - \mu_t - c_t) l(0) = \mu_t l(1) + (1-\mu_t) l(0) + (l(1) - l(0)) c_t  \\
    &= \mu_t^* l(1) + (\mu_t - \mu_t^*) l(1) + (1 - \mu_t^*) l(0) - (\mu_t - \mu_t^*) l(0) + (l(1) - l(0)) c_t  \\
    &= \Eb_{\mu_t^*}[l(Y)] - (l(1) - l(0)) \Eb_{\mu_t^*} [Y - \mu] + (l(1) - l(0)) c_t = L^{\mathrm{pay}}_t 
\end{align*}
\end{proof}

\standardcalibration*
\begin{proof}
If the MCE error of $\mu$ is less than $c_0$, denote $U = \mu(X)$ by definition we have, for every $U \in [0, 1]$
\begin{align*}
    | U - \Eb[Y \mid U] | \leq c_0 \numberthis\label{eq:standard_calibration_1}
\end{align*}
For any $b \in \Bc$, denote $b(X) := \tg(\mu(X)) = \tg(U)$ we have
\begin{align*}
    \Eb[ b(X) (\mu(X) - Y) - |b(X)| c(X) ] &= \Eb\left[ \Eb[\tg(U) (\mu(X) - Y) - | \tg(U)|  c_0 \mid U] \right] & \text{Iterated\ Expectation} \\ 
    &= \Eb[ \tg(U) \Eb[\mu(X) - Y \mid U] - | \tg(U) |c_0 ] & \Eb[UZ \mid U] = U \Eb[Z \mid U]\\
    &= \Eb [\tg(U) (U - \Eb[Y \mid U]) - | \tg(U) |c_0 ] & \text{Linearity} \\
    &\leq \Eb[ |\tg(U)| |U - \Eb[Y \mid U]| - |\tg(U)| c_0 ] & \text{Cauchy\ Schwarz} \\
    &= \Eb[ |\tg(U)| \left( | U - \Eb[Y \mid U] |  - c_0 \right) ] \leq 0 & \text{By\ Eq.(\ref{eq:standard_calibration_1}})
\end{align*}
which shows that $\mu, c$ is sound. 

Conversely suppose there is some interval $(u_0, u_1)$ such that whenever $U \in (u_0, u_1)$
\begin{align*}
    U - \Eb[Y \mid U] > c_0
\end{align*}
we can choose $b(X) := \tg(U) = \mathbb{I}(U \in [u_0, u_1])$ we have
\begin{align*}
    \Eb[ b(X) (\mu(X) - Y) - |b(X)| c(X) ]  &=  \Eb[ |\tg(U)| \left( | U - \Eb[Y \mid U] |  - c_0 \right) ] > 0
\end{align*}
so the forecaster is not sound.
We can show a similar proof when 
\begin{align*}
    U - \Eb[Y \mid U] < -c_0
\end{align*}
\end{proof}

\multicalibration*
\begin{proof}
Denote $U = \mu(X)$. Suppose $\mu, c$ is not multi-calibrated, then there exists $S \in \Sc$ and there exists some interval $(u_0, u_1)$ such that whenever $U \in (u_0, u_1)$
\begin{align*}
   | \mathbb{I}(X \in S) (U - \Eb[Y \mid U]) | > c_0
\end{align*}
Suppose $\mathbb{I}(X \in S) (U - \Eb[Y \mid U]) > c_0$
we can choose $b(X) := \mathbb{I}(U \in [u_0, u_1] \cap X \in S)$ we have
\begin{align*}
    \Eb[ b(X) (\mu(X) - Y) - |b(X)| c(X) ]  &=  \Eb[ |b(X)| \left( | U - \Eb[Y \mid U] |  - c_0 \right) ] > 0
\end{align*}
We can show a similar proof when $\mathbb{I}(X \in S) (U - \Eb[Y \mid U]) < -c_0$.
\end{proof}

\end{document}